\documentclass[letterpaper,10pt]{article}
\usepackage{todonotes}
\usepackage{float}

\usepackage{probml}

\usepackage{preambles/general_preamble__probml}

\ShortHeadings{Diffusion SSM}

\usepackage{blindtext}
\usepackage{wrapfig}

\usepackage{bm,amssymb,amsmath}

\makeatletter
\def\th@plain{%
  \thm@notefont{}
  \itshape 
}
\def\th@definition{%
  \thm@notefont{}
  \normalfont 
}
\makeatother

\def\1{\bm{1}}

\DeclareMathAlphabet{\mathsfit}{\encodingdefault}{\sfdefault}{m}{sl}
\SetMathAlphabet{\mathsfit}{bold}{\encodingdefault}{\sfdefault}{bx}{n}

\newcommand{\E}{\mathbb{E}}

\newcommand{\KL}{D_{\mathrm{KL}}}

\usepackage{diffusion} 

\let\OldKL\KL
\renewcommand{\KL}[2]{\OldKL\!\left(#1 \,\middle\|\, #2\right)}
\begin{document}

\title{Diffusion-Driven State Space Models}

\author[1,$*$]{Jack Ruder}
\author[1,$\dagger$]{Michael Wojnowicz}

\affil[1]{Montana State University}
\affil[$*$]{\url{jackruder@montana.edu}}
\affil[$\dagger$]{\url{michael.wojnowicz@montana.edu}}

\maketitle

\begin{abstract}
    In many domains, practitioners seek models that produce accurate forecasts while faithfully
      capturing latent system dynamics.
    Existing approaches typically sacrifice one of these goals: deep state space models often assume
      Gaussian latent transitions, limiting fit and forecasting, while diffusion models are highly
      expressive but lack principled inference for the underlying dynamics.
    To combine the strengths of both, we introduce the Diffusion-Driven State Space Model (DDSSM),
      which replaces the conventional Gaussian transition distribution with a diffusion model.
    Our DDSSM resolves the open problem of how to jointly train an autoencoder and a diffusion model on
      sequential data, thereby extending the literature on latent diffusion models for time series.
    Moreover, we find that the DDSSM empirically outperforms a state-of-the-art deep SSM at fitting and
      forecasting a simulated time series with multimodal transitions.
\end{abstract}


\section{Introduction} \label{sec:intro}

Diffusion models~\citep{sohl2015deep, ho2020denoising} have emerged as a powerful class of
  generative models for high-dimensional data.
For time series data, emerging evidence demonstrates state-of-the-art performance on forecasting,
  imputation, and sample-generation tasks~\citep{yang2024survey}, where diffusion models are
  typically used as generative samplers to produce future trajectories conditioned on past
  observations~\citep{su2025diffusion}.
Some approaches perform diffusion directly in observation space~\citep{tashiro2021csdi,
      rasul2021autoregressive, kollovieh2023predict, price2025probabilistic},
\AFTER{Add more cites to diffusion in latent space.
    Some are referenced in \url{https://openreview.net/pdf?
    id=nAyeE7cAS0.}  }while others apply diffusion in the latent space of an autoencoder~\citep{feng2024latent,
    liu2024align, Blattmann_2023_CVPR}, which can improve tractability for high-dimensional data.
However, existing latent time-series diffusion methods typically rely on two-stage training: first
  learning the autoencoder, then training the diffusion model.
This may produce an over-regularized latent space~\citep{liu2024align} that discards information
  needed for downstream generative dynamics, consistent with research on latent diffusion for static
  data where two-stage training leads to poorer representations and generative
  performance~\citep{vahdat2021score, shmakov2023end}.


In contrast, State Space Models (SSMs)~\citep{kalman1960new,shumway1982approach,durbin2001time}
  provide a classical framework for temporal data that naturally supports end-to-end training,
  captures timestep-by-timestep dynamics, and admits well-studied extensions to diverse settings.
SSMs typically assume observed data are generated from latent states that evolve over time via a
  Markov process, with each observation depending on the current state.
Classically, the linear Gaussian state-space model assumes first-order Markovian linear dynamics
  and linear Gaussian observation models~\citep{roweis1999unifying}.
To absolve linearity, transition and emission densities may be parameterized with deep neural
  networks~\citep{krishnan2015deep, krishnan2017structured}, where training is performed via
  stochastic backpropagation~\citep{rezende2014stochastic, kingma2014autoencoding}.
These neural-network-parameterized SSMs may be referred to as Deep SSMs, or Dynamical Variational
  Autoencoders (DVAEs); see~\citet{girin2022dynamical} for a review.
\AFTER{compare with neural ODEs/SDEs, which are also Deep SSMs.
    Make a distinction to structured SSMs (S4/MAMBA) which are not probabilistic models.
}
While neural-networks allow the mean of each latent state to be a \textit{nonlinear} function of
  the previous latent state, the transition densities in these models assume some parametric
  form~\citep{karl2016deep}, most often Gaussian~\citep{krishnan2015deep, johnson2016composing,
      fraccaro2017disentangled, krishnan2017structured,ansari2021deep, girin2022dynamical,
      revach2022kalmannet, dowling2024exponential, hashempoorikderi2024gated, chinellato2025state},
  thereby limiting their expressivity.
Moreover, Gaussian transition densities induce a Gaussian prior on the latent trajectories.
Within the VAE framework the Gaussian prior is known to over-regularize, yielding latent
  representations that poorly represent the structure of the data~\citep{tomczak2018vae,
      klushyn2019learning, chen2017variational}.
\TODO{Relate this over-regularization more explicitly to posterior collapse.
    Are they the same?
}
\AFTER{Add even more cites to deep SSMs with non-Gaussian transitions.
    I think I should know of a bunch of examples from my work in switching SSMs.
}


Inspired by the framework of composing probabilistic graphical models with neural
  networks~\citep{johnson2016composing}, we propose to get the ``best of both worlds'' by composing
  SSMs with diffusion models. Our \textit{diffusion-driven} SSMs (DDSSMs) naturally model the real-time dynamics of a system
  without imposing restrictive Gaussian assumptions about its latent dynamics. Further, we adapt SSM ideas to resolve the open problem of how to jointly train an autoencoder and a diffusion model on sequential data in a principled manner. In this way, we extend latent diffusion models for time series.

\AFTER{Here we could refer specifically to the form of the KL regularization term in the ELBO.)}  \TODO{Consider adding  other  benefits and/or why joint training will have further downstream benefits.}



\section{Method} \label{sec:method}

\textbf{Generative Model.}
Let \( \mathbf{x}_{1:T} = (\mathbf{x}_1, \dots, \mathbf{x}_T) \) with \( \mathbf{x}_t \in
  \mathbb{R}^D \) be a time series of \(T\) observations, and let \( \mathbf{u}_{1:T} =
  (\mathbf{u}_1, \dots, \mathbf{u}_T) \) with \( \mathbf{u}_t \in \mathbb{R}^V \) denote covariates
  observed at all time steps.
We assume the observations $\mathbf{x}_{1:T}$ are generated from latent variables \(
  \mathbf{z}_{1:T} = (\mathbf{z}_1, \dots, \mathbf{z}_T) \) with \( \mathbf{z}_t \in \mathbb{R}^M \)
  that capture the underlying system dynamics.
After applying standard conditional independence assumptions,
the density of the complete data likelihood \( p(\mathbf{x}_{1:T}, \mathbf{z}_{1:T} | \mathbf{u}_{1:T}) \) can be factorized in the form of a State Space Model (SSM) as
\begin{align}
    \label{eq:kdiff}
    p(\mathbf{x}_{1:T}, \mathbf{z}_{1:T} | \mathbf{u}_{1:T})
    =
     & \underbrace{
        p_{\eta,\theta}(\mathbf{z}_{1:j}, \mathbf{x}_{1:j} | \mathbf{u}_{1:j})
    }_{\mathclap{\text{State Initialization}}}
    \prod_{t=j+1}^{T}
    \underbrace{
    p^^t_{\psi}(\mathbf{z}_t | \mathbf{z}_{t-j:t-1}, \mathbf{u}_{t-j: t })\;
    }_{\mathclap{\text{Transition Density}}} \,
    \underbrace{
    p^^t_{\theta}(\mathbf{x}_t | \mathbf{z}_{t-j+1:t}, \mathbf{u}_{t-j + 1 : t })
    }_{\mathclap{\text{Emissions Density}}},
\end{align}
where \(\theta\) denotes the parameters of the emissions density, \(\psi\) the parameters of the transition density, and \(\eta\)  the parameters of the initialization model.
The hyperparameter \(j\) determines the size of the conditioning set for the transition and
  emissions densities, and in particular the Markov order of the latent sequence.
The superscript $t$ designates that the emissions and transitions are time-inhomogeneous.
See \cref{app:model_derivation} for the derivation and assumptions leading to this factorization.


As mentioned in the introduction, most existing deep SSMs model the transitions with Gaussian
  densities where the means and variances are parameterized by neural networks.
We depart from this assumption by modeling the transition density \(p^^t_\psi\) with a diffusion
  model, which allows us to learn arbitrary transition densities in the latent space.
Our diffusion-driven transition model induces a wider prior over the latent trajectories, allowing
  us to learn more flexible latent representations of the data.
The diffusion model introduces a sequence of intermediate (increasingly noisy) latent states \( (\mathbf{z}_t^1, \mathbf{z}_t^2, \ldots, \mathbf{z}_t^K) \) for each time step \(t\),  giving an augmented transition model
\begin{align}
    p^^t_\psi(\mathbf{z}_t^0, \mathbf{z}_t^{1:K} | \mathbf{z}_{t-j:t-1}, \mathbf{u}_{t-j: t }) & =
    p(\mathbf{z}_t^K) \prod_{k=1}^K p^^t_\psi(\mathbf{z}_t^{k-1} | \mathbf{z}_t^k,  \mathbf{z}_{t-j:t-1}, \mathbf{u}_{t-j: t }), \label{eq:transition_augmented}
\end{align}
where $\mathbf{z}_t^0 \defeq \mathbf{z}_t$ and $\mathbf{z}_t^K$ is approximately distributed as $\N(\+0,\+I)$.\TODO{\cref{app:diffusion_background} should state the explicit form of the steps in the denoising process, like Equation 1 in the DDPM paper.
Then we can direct the reader to that here.
}
Marginalizing \( \mathbf{z}_t^{1:K} \) recovers the transition density in \cref{eq:kdiff} while
  enabling complex, multimodal transitions beyond Gaussian dynamics.
Following~\citep{tashiro2021csdi}, we include the time index $t$ and the diffusion noise level $k$
  as conditioning inputs, such that a single neural network with parameters $\psi$ can parameterize
  the family of densities $(p^^t_\psi)_{t=1}^T$ across all time steps and noise levels\footnotemark.
\footnotetext{Some works instead include the time
    \textit{difference} as a covariate, in which case the transition model is
    time-homogeneous~\citep{krishnan2017structured}.}

For simplicity, we assume Gaussian emissions, with the diffusion model handling non-Gaussian
  structure and the emissions capturing any remaining uncertainty.
We choose the standard
  parameterization\TODO{Is it standard to take time as an input?} of the emissions as \(p^^t_\theta(\mathbf{x}_t
  | \mathbf{r}_t) = \mathcal{N}(\mu_\theta(\mathbf{r}_t), \Sigma_\theta(\mathbf{r}_t))\), where
  \(\mu_\theta\) and \(\Sigma_\theta\) are neural networks mapping recent latent states and
  covariates $\mathbf{r}_t = (\mathbf{z}_{t-j+1:t}, \mathbf{u}_{t-j + 1 : t })$ to the mean and
  covariance parameters of a Gaussian distribution in observation space.
As with the transitions, a single neural network can be used to parameterize the emissions at all
  time steps by including the time index \(t\) within the covariates $\+u_t$.
We model the initialization density via the factorization \(p_{\eta,\theta}(\mathbf{z}_{1:j},
  \mathbf{x}_{1:j} | \mathbf{u}_{1:j}) = p_\eta(\mathbf{z}_{1:j} | \mathbf{u}_{1:j})
  p^^t_\theta(\mathbf{x}_{1:j} | \mathbf{z}_{1:j}, \mathbf{u}_{1:j})\).
Then, \(p_\eta(\mathbf{z}_{1:j} | \mathbf{u}_{1:j})\) is made expressive using a variational
  hierarchical prior (VHP)~\citep{klushyn2019learning, klushyn2021latent}---as demonstrated to be
  effective for deep SSMs.
\AFTER{Replace VHP with diffusion model}
See \cref{app:vhp} for further discussion and implementation details.


\textbf{Inference.}
As the true posterior density \( p(\mathbf{z}_{1:T}, \{\mathbf{z}_t^{1:K}\}_{t=j+1}^T |
  \mathbf{x}_{1:T}, \mathbf{u}_{1:T}) \) is intractable, we approximate it with a variational density
  \(q_\phi(\mathbf{z}_{1:T}, \{\mathbf{z}_t^{1:K} \}_{t=j+1}^T | \mathbf{x}_{1:T}, \mathbf{u}_{1:T})
  \) with learnable parameters \(\phi\) .
Our \cref{lem:dsep} (see \cref{app:post_fact}) along with the well-known fact that the optimal variational posterior factorizes identically to the true posterior~\citep{wainwright2008graphical}  implies that we can assume the following factorization for the variational posterior density \( q_\phi(\mathbf{z}_{1:T}, \{\mathbf{z}_t^{1:K}\}_{t=j+1}^T | \mathbf{x}_{1:T}, \mathbf{u}_{1:T}) \) without making any approximation other than the dependence on $\phi$:
\begin{equation}
    \label{eq:var_post_main_body}
    \begin{split}
         & q_\phi(\mathbf{z}_{1:T},
        \{\mathbf{z}_t^{1:K}\}_{t=j+1}^T | \mathbf{x}_{1:T}, \mathbf{u}_{1:T})
        =
        \underbrace{
        q_\phi\!\left(\mathbf{z}_{1 : j} | \mathbf{x}_{1 : T}, \mathbf{u}_{1 : j}\right) }_{\mathclap{\text{State Initialization}}} \\
         & \;\times\;
        \underbrace{\prod_{t=j+1}^{T}
            q^^t_\phi\!\left(\mathbf{z}_t |\, \mathbf{z}_{t-j:t-1}, \mathbf{x}_{t:T}, \mathbf{u}_{t-j  : t} \right)
        }_{\mathclap{\text{Smoothed Transitions}}}
        \;\times\;
        \underbrace{\prod_{t=j+1}^{T}
        q_{\text{chain}}(\mathbf{z}_t^{1:K} | \mathbf{z}_t)
        }_{\mathclap{\text{Diffusion Noising Process}}}.
    \end{split}
\end{equation}
Notably, this factorization contains the Kalman smoothing factor \(q^^t_\phi\left(\mathbf{z}_t |
  \mathbf{z}_{t-j : t-1}, \mathbf{x}_{t:T}, \mathbf{u}_{t-j : t}\right)\), so the optimal encoder's
  latent state at time \(t\) depends on the \textit{past} latent states and covariates, but
  \textit{future} observations.
\TODO{Here give a few cites for contrast from the literature on latent diffusion for time series where this assumption was not made (but could have been for some gain), or was explicitly violated.}
Also note that as with models for static data, the diffusion noising process does not require
  learning.

We jointly learn the variational parameters $\phi$ along with our model's generative parameters
  $(\theta, \psi, \eta)$ by maximizing an evidence lower bound (ELBO) on the marginal log-likelihood
  \( \log p(\mathbf{x}_{1:T} | \mathbf{u}_{1:T}) \).
The ELBO for our model is
\begin{subequations} \label{eq:elbo_full_body}
    \begin{align}
        \log p\left(\mathbf{x}_{1:T} | \mathbf{u}_{1 : T}\right) \nonumber
         & \geq - \underbrace{ \KL{q_\phi(\mathbf{z}_{1 : j} | \mathbf{x}_{1: T}, \mathbf{u}_{1: T})}{p_{\eta}(\mathbf{z}_{1:j} | \mathbf{u}_{1:j})}
        }_{\text{KL Divergence for Initialization}} \\[-4pt]
         & + \underbrace{\sum_{t=1}^T \mathbb{E}_{q_\phi(\mathbf{z}_{1:T} | \mathbf{x}_{1:T}, \mathbf{u}_{1 : T})} \Big[ \log p^^t_\theta\left(\mathbf{x}_t | \mathbf{z}_{\max(t-j+1,1) : t}, \mathbf{u}_{\max(t-j + 1, 1) : t}\right)\Big]}_{\text{Negative Reconstruction Error}} \label{eq:term_seq_reconstruct_body} \\[-4pt]
         & - \underbrace{\sum_{t=j+1}^T \KL{q^^t_\phi(\mathbf{z}_t | \mathbf{z}_{t-j:t-1}, \mathbf{x}_{1:T}, \mathbf{u}_{1 : T})}{p^^t_\psi(\mathbf{z}_t  | \mathbf{z}_{t-j:t-1}, \mathbf{u}_{t-j: t })}}_{\text{KL Divergence for Transition}}, \label{eq:term_kl_trans} \\[-\belowdisplayskip\vspace{-1em}] \nonumber
    \end{align}
\end{subequations}
where \cref{eq:term_kl_trans} can be further rewritten as
\begin{subequations}
    \begin{align}
         & \underbrace{\sum_{t=j+1}^T \mathbb{E}_{q_\phi(\mathbf{z}_{1:T}, \left\{\mathbf{z}_t^{1:K}\right\}_{t=j+1}^T | \mathbf{x}_{1:T}, \mathbf{u}_{1 : T})} \Big[\log p(\mathbf{z}_t^K) + \sum_{k=1}^K -\log q(\mathbf{z}_t^{k} | \mathbf{z}_t^{k-1}) +\log p^^t_\psi(\mathbf{z}_t^{k-1} | \mathbf{z}_t^k, \mathbf{c}_t)\Big]}_{\text{Diffusion Chain}} \label{eq:term_diff_chain_body} \\[-4pt]
         & \qquad\qquad\qquad\qquad+ \underbrace{\sum_{t=j+1}^T \mathbb{E}_{q_\phi(\mathbf{z}_{1:T} | \mathbf{x}_{1:T}, \mathbf{u}_{1 : T})} \Big[- \log q^^t_\phi\left(\mathbf{z}_t | \mathbf{z}_{t-j : t-1},  \mathbf{x}_{t:T}, \mathbf{u}_{t-j : t}\right)\Big]}_{\text{Posterior Transition Entropy}}, \label{eq:post_entropy_body} \\[-\belowdisplayskip\vspace{-1em}] \nonumber
    \end{align}
\end{subequations}
with \(\mathbf{c}_t = (\mathbf{z}_{t-j:t-1}, \mathbf{u}_{t-j: t })\) denoting the conditioning set for the diffusion model at time \(t\).
\TODO{Do the q's need t superscripts?}
Unlike existing latent time-series diffusion models\TODO{TODO: Needs cites}\AFTER{We should back
  this up empirically} that model \(q_\phi(\mathbf{z}_{1:T} | \mathbf{x}_{1:T}, \mathbf{u}_{1:T})\)
  without factorization, our objective mirrors the conditional independence structure of the true
  posterior, providing a strong inductive bias that narrows the ELBO approximation gap.
Specifically per \cref{eq:var_post_main_body}, the encoder constructs latent representations for
  \(\mathbf{z}_t\) by connecting the past (via \(\mathbf{z}_{t-j :t-1}\)) to the future (via
  \(\mathbf{x}_{t:T}\)).
The transition model does not observe \(\mathbf{x}_{t : T}\), thus by jointly training the
  transition model and the encoder, \cref{eq:term_kl_trans} encourages the encoder to produce latent
  states that are predictable under the dynamics learned by the transition model.
Moreover, as our diffusion model is highly expressive, the regularization of \(\mathbf{z}_t\) is no
  longer constrained towards a strict Gaussian prior.
\vspace{-1.5em}
\paragraph{Parameterization and Training.}
The design of the neural networks parameterizing these densities is detailed in \cref{apd:model}.
We follow standard practice~\citep{girin2022dynamical} in minimizing the negative ELBO with
  stochastic backpropagation.
To ensure the flow of gradients through the diffusion model to the VAE are normalized across noise
  levels, we use preconditioning~\citep{karras2022elucidating} as implemented in
  \cref{app:diff_reparameterization} to rewrite \cref{eq:term_diff_chain_body}.
Our full training algorithm, along with an explanation of how to compute each of the terms in the
  ELBO, can be found in \cref{app:implementation}.

\section{Related work}
\TODO{Create a section which directly explains how our method abstracts entire/special cases of a number of related works.
    This may need to be done post-probml, and regardless would likely need to live in the appendix.
    We would need to describe identity emissions, differences in how inference is performed,
      differences in how the latent space is learned, differences in the training objective, etc. }

      \TODO{Need to add some comments about use of NFs and related, per Mike H.
    comment about other ways to introduce expressivity.}

\textbf{Deep state space models.}
Similarly to other Deep State Space Models (DSSMs)~\citep{girin2022dynamical} such as the Deep
  Kalman Filter (DKF)~\citep{krishnan2015deep, krishnan2017structured}, our Diffusion-Driven State
  Space Model (DDSSM) parametrizes the transitions and emissions densities of a state space model
  with a neural network.
However, our model replaces the parametric (typically Gaussian) transition distribution of a DSSM
  with a diffusion model.
This design enables complex, potentially multi-modal transition dynamics (validated in
  \cref{sec:simulation}) and avoids the restrictive regularization of Gaussian
  transitions~\citep{karl2016deep}.
Integration with advances such as the extended Kalman filter VAE~\citep{klushyn2021latent} is left
  for future work.


\textbf{Diffusion models for temporal data.}
Some works~\citep{rasul2021autoregressive,tashiro2021csdi,kollovieh2023predict} apply diffusion
  directly in observation space, but the high computational cost of diffusion -- motivating latent
  diffusion for static data~\citep{rombach2022high} -- is exacerbated in time series due to the need
  to model both high-dimensional structure and temporal dynamics.
This has led to latent diffusion approaches for time
  series~\citep{qian2024timeldm,suh2025timeautodiff,liu2024align,feng2024latent}, which typically
  rely on two-stage training (autoencoder then diffusion), potentially yielding over-regularized
  latent spaces that discard information needed for generative
  dynamics~\citep{liu2024align,vahdat2021score,shmakov2023end}.
Our approach instead follows~\citet{vahdat2021score} and~\citet{shmakov2023end} by training
  diffusion end-to-end in latent space, extended here to the temporal setting.
For further details on related works, see~\cref{sec:extended_related_works}.

\TODO{Need to add some comments about use of NFs and related, per Mike H. comment about other ways to introduce expressivity.}

\AFTER{Create a section which directly explains how our method abstracts entire/special cases of a number of related works.
    This may need to be done post-probml, and regardless would likely need to live in the appendix.
    We would need to describe identity emissions, differences in how inference is performed,
    differences in how the latent space is learned, differences in the training objective, etc. }


\section{Simulation study: Bimodal Noise} \label{sec:simulation}
To demonstrate the advantage of flexible transitions, we implement a latent dependent-step random walk with bimodal noise
\begin{align*}
    z^{(i)}_t & = 0.9 z^{(i)}_{t-1} + 4s^{(i)}_t, \quad s^{(i)}_t \sim \text{Unif}(\{-1, 1\}). \\
    x^{(i)}_t & = z^{(i)}_t + \epsilon^{(i)}_t,   \qquad\quad\: \epsilon^{(i)}_t \sim \mathcal{N}(0, 0.2),
\end{align*}
where \(i=1,\ldots,N\) indexes i.i.d sequences, and \(t=1, \ldots, T\) indexes the time steps of those sequences.
We fix \(T=32\) and generate \(N=1024\) distinct sequences each for training, validation, and
  testing.

We fit three models to the data, fixing a latent dimension of 1.
The first is the DKF equipped with a VHP, and second is our DDSSM model.
As a na\"ive baseline we include a Last Observation Carried Forward (LOCF)
  strategy~\citep{shao2003last}, which carries forward the last observed value, i.e $\hat{x}^{(i)}_t
      = x^{(i)}_{t-1}$.
Details of the exact training procedure are in \cref{app:implementation}; model configurations and
  hyperparameter search procedure are provided in \cref{app:experiments}.
\begin{figure}[ht] \label{fig:jsd-report}
    \centering
    \setlength{\belowcaptionskip}{-6pt}
    \begin{subfigure}[b]{0.48\textwidth}
        \centering
        \includegraphics[width=\textwidth]{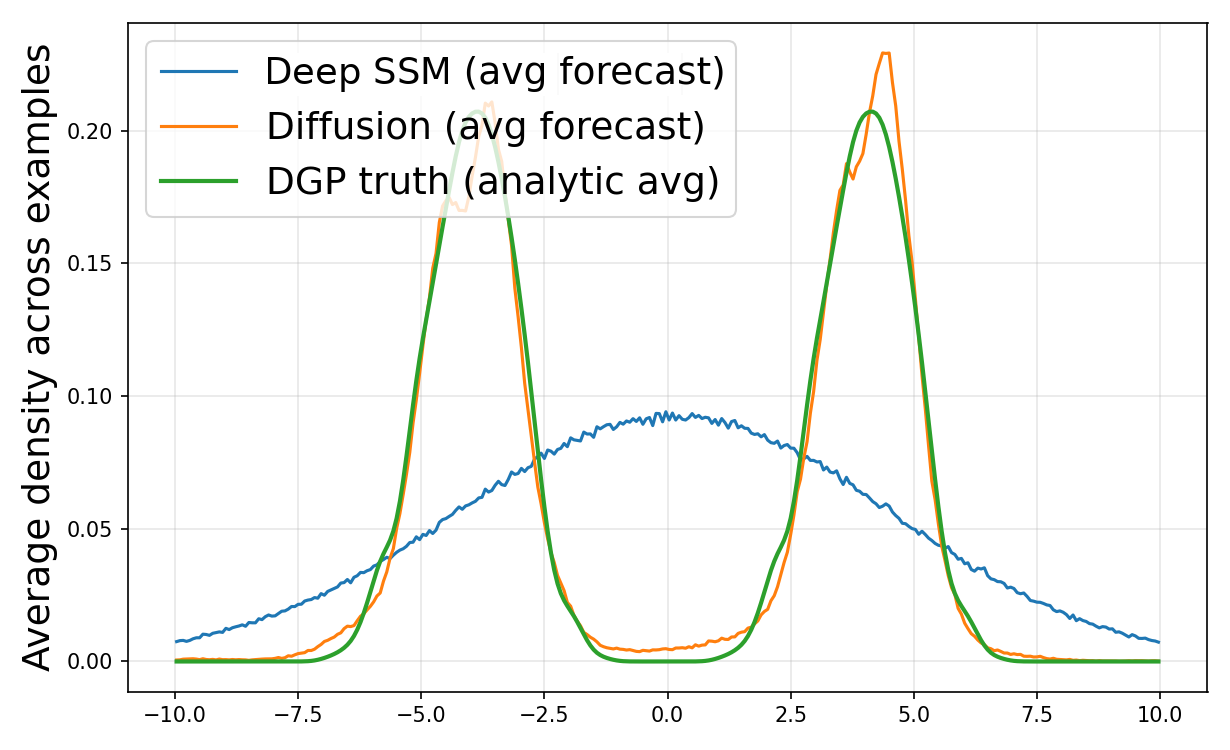}
        \caption{Predictive distributions, centered on \(x_{31}\)} \label{fig:jsd-report-preds}
    \end{subfigure}
    \begin{subfigure}[b]{0.48\textwidth}
        \centering
        \includegraphics[width=\textwidth]{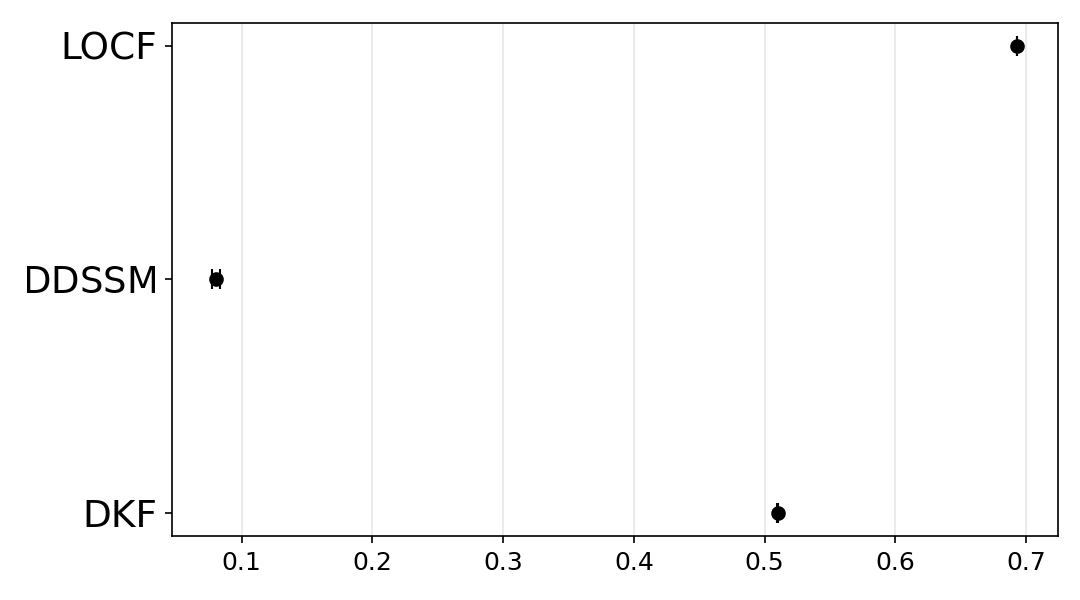}
        \caption{Jensen-Shannon Divergence} \label{fig:jsd-report-jsd}
    \end{subfigure}
    \caption{ On the left, the average predictive distribution of each model across the test set is plotted against the true distribution.
        On the right, the JSD is plotted as the mean across the test set for each model, with (small) error
          bars giving $\pm $ 1 standard error of the mean.
    }
\end{figure}
To evaluate the performance of the models,~we draw 1024 samples from each model.
Then, we compute the Jensen-Shannon Divergence (JSD) between the 1-step predictive distribution of
  each model and the true distribution of \(x^{(i)}_{32}\) given \(x^{(i)}_{1:31}\) for each example
  in the test set.
Each sample is drawn by first drawing a latent sample \(\hat{z}^{(i)}_{31}\) from the encoder, and
  then sampling the transition and emission distributions to produce \(\hat{x}^{(i)}_{32}\).
\Cref{fig:jsd-report-preds} reveals, as expected, that the DKF places probability mass
between the two modes of the true distribution, while the DDSSM is able to capture both modes.
Qualitatively, this results in a significantly lower JSD as shown in \cref{fig:jsd-report-jsd}.
\AFTER{Can you briefly explain \textit{why} the DKF handles the bimodal transition by covering both modes rather than choosing one?
    Does this trace back to an (explicit or implicit) inclusive KL term (as opposed to an exclusive KL
      term) in the objective?
    If so, it make take a few sentences or a paragraph to fully explain, but you could do that in the
      Appendix, summarize in 0.5-1 sentences here, and give a link to the relevant part of the Appendix.
}
To investigate the impacts of diffusion transitions on prediction, we plot the reconstructions and
  forecasts of two representative samples from the test set, chosen according to the median value of
  the JSD metric for the DDSSM model.
\AFTER{ We could quantify reconstruction MSE, by directly reporting the reconstruction loss (these are 1-1).
    Quality of "forecast" will probably need something else.
    Maybe see the D3SSM paper (I think that's the one, it does deep switching).
}
\Cref{fig:sim-results} shows that the DDSSM's trajectories are representative of the data, while the DKF's forecasts do not have a consistent pattern.
Notably, the DKF fails to provide good reconstructions of the observed data, supporting our claim
  that a Gaussian prior over the latent trajectories is harmful for learning latent
  representations.\TODO{What are the real-world implications of getting the \textit{form} of the
  predictive distribution correct (as opposed to simply the mean)?
}
\begin{figure}[H]
    \centering
    \setlength{\belowcaptionskip}{-6pt}
    \begin{subfigure}[b]{0.48\textwidth}
        \centering
        \includegraphics[width=\textwidth]{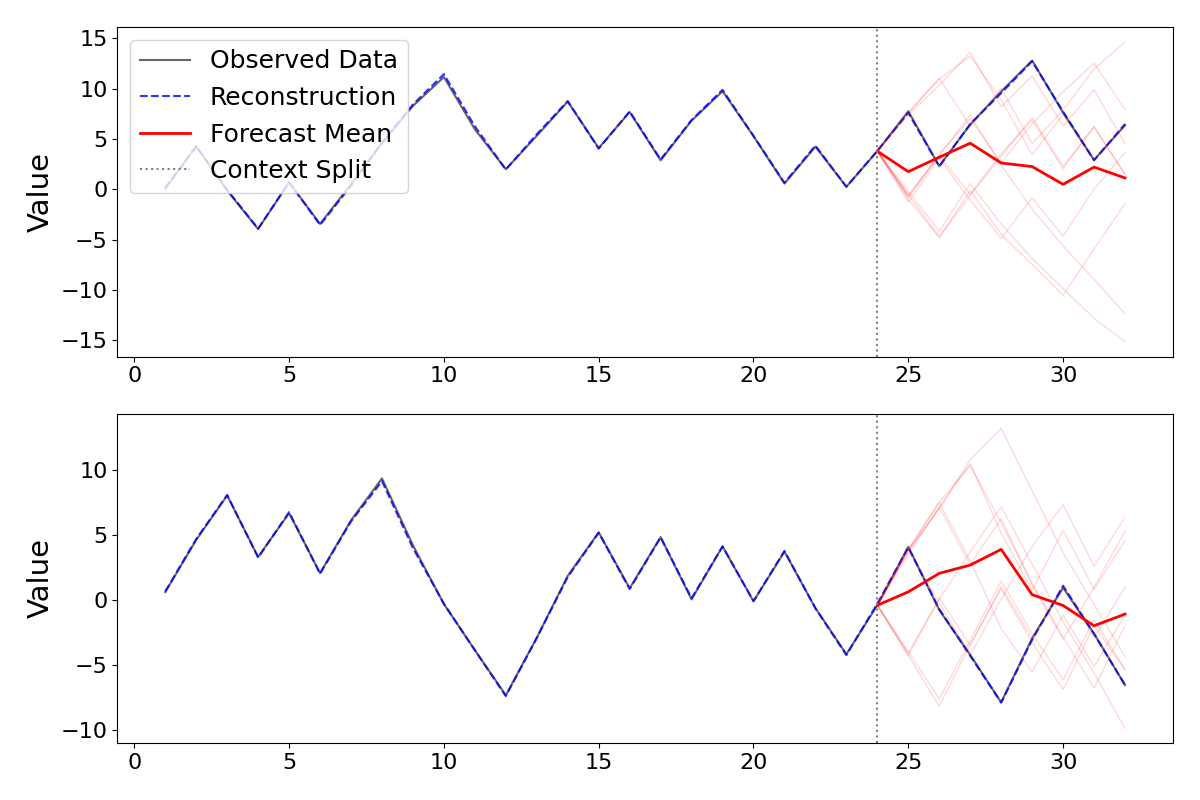}
        \caption{DDSSM} \label{fig:sim-results-diff}
    \end{subfigure}
    \begin{subfigure}[b]{0.48\textwidth}
        \centering
        \includegraphics[width=\textwidth]{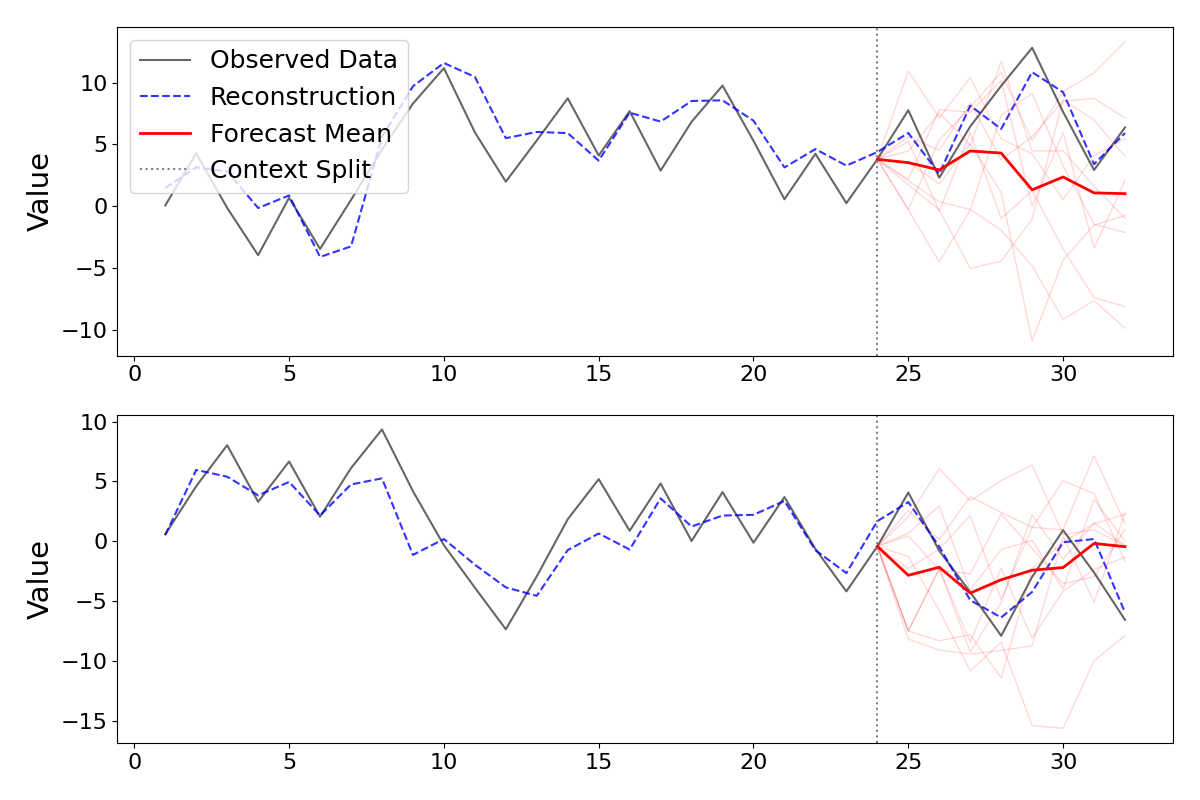}
        \caption{DKF} \label{fig:sim-results-gauss}
    \end{subfigure}
    \caption{
        Reconstructions and Forecasts.
        The dashed line shows \(\hat{x}^{(i)}_{1:32}\) for each model, computed by sampling
          \(\hat{z}^{(i)}_{1:32}\) from the encoder, and subsequently sampling \(\hat{x}^{(i)}_{1:32}\) from
          the decoder.
        Given the encoder's sample \(\hat{z}^{(i)}_{24}\), we generate trajectories of
          \(\hat{x}^{(i)}_{25:32}\) by sampling from the transitions and emissions.
        Trajectories are shown in light red, and the mean of these trajectories is in solid red.
    } \label{fig:sim-results}
\end{figure}


\section{Discussion}
In this work we have demonstrated that we can formulate latent diffusion models for time series by
  casting the problem as a state space model with a diffusion transition distribution.
This allows end-to-end training, allowing flexible reconstructions without penalizing the learning
  of the latent dynamics.
We aim to continue our work by applying the DDSSM to large-scale datasets, and by incorporating
  ideas from deep SSM literature to fully utilize the power of the diffusion transition distribution.



\clearpage

\section*{Acknowledgements}
\textbf{Funding.}
This material is based upon work supported in part by the National Science Foundation EPSCoR
  Cooperative Agreement OIA-2242802.
Any opinions, findings, and conclusions or recommendations expressed in this material are those of
  the author(s) and do not necessarily reflect the views of the National Science Foundation.

\textbf{Computing Resources.}
Computational efforts were performed on the Tempest High Performance Computing System, operated and
  supported by University Information Technology Research Cyberinfrastructure (RRID:SCR\_026229) at
  Montana State University.

\bibliography{diffusion}

\clearpage

\appendix

\section{State Space Model Derivation}
\subsection{Derivation of the Generative Model} \label{app:model_derivation}
Our goal is to learn a model of the conditional distribution \( p(\mathbf{x}_{1:T} |
  \mathbf{u}_{1:T}) \).
Following the general framework of~\citet{girin2022dynamical},
we introduce a sequence of latent variables \( \mathbf{z}_{1:T} = (\mathbf{z}_1, \mathbf{z}_2, \ldots, \mathbf{z}_T) \), where each \( \mathbf{z}_t \in \mathbb{R}^d \) is a \( d \)-dimensional vector,
and factorize the joint distribution \( p(\mathbf{x}_{1:T}, \mathbf{z}_{1:T} | \mathbf{u}_{1:T}) \) as
\begin{equation}
    \label{eq:genDVAE}
    p(\mathbf{x}_{1:T}, \mathbf{z}_{1:T} | \mathbf{u}_{1:T}) = \prod_{t=1}^T p(\mathbf{x}_t|\mathbf{x}_{1: t-1}, \mathbf{z}_{1: t}, \mathbf{u}_{1 : t})p(\mathbf{z}_t | \mathbf{x}_{1: t-1}, \mathbf{z}_{1: t-1}, \mathbf{u}_{1 : t}).
\end{equation}
We make the Markovian assumption that the latent state at time \(t\) only depends on a fixed-length history of \(j\) previous latent states and covariates, rather than the entire history, i.e.
\begin{equation}
    \label{eq:markov-latent}
    p\left(\mathbf{z}_t | \mathbf{z}_{1 : t-1} , \mathbf{u}_{1 : t}\right) = p\left(\mathbf{z}_t | \mathbf{z}_{t-j : t-1}, \mathbf{u}_{t-j : t}\right),
\end{equation}
for \(t > j\) where \(j\) is a fixed Markov order.

On the emissions model, we assume that the observation at time \(t\) only depends on a fixed-length history of \(j\) previous latent states and covariates, i.e.
\begin{equation}
    \label{eq:indep-emission}
    p(\mathbf{x}_t | \mathbf{x}_{1: t-1}, \mathbf{z}_{1: t}, \mathbf{u}_{1 : t }) = p(\mathbf{x}_t | \mathbf{z}_{t - j + 1: t}, \mathbf{u}_{t-j + 1 : t }).
\end{equation}
We explicitly choose the same conditioning set for the emissions model, arguing that the most
  recent \(j\) states carry the full information about the system's current state, and thus the most
  recent \(j\) states and covariates should be sufficient for predicting the current observation.
Alternatively we could choose a smaller order on the emissions as well; this is not a key aspect of
  our model and leave this choice to future work.

Lastly, we assume that the initial \(j\) latent states are generated from a separate initial state distribution, i.e.
\begin{equation}
    \label{eq:init-state}
    p(\mathbf{z}_{1:j}, \mathbf{x}_{1:j} | \mathbf{u}_{1:j}) = p_{\eta}(\mathbf{z}_{1:j} | \mathbf{u}_{1:j}) p_{\theta}(\mathbf{x}_{1:j} | \mathbf{z}_{1:j}, \mathbf{u}_{1:j}).
\end{equation}

Rewriting \cref{eq:genDVAE} with the above assumptions, we have
\begin{equation}
    \begin{aligned}
    \label{eq:kdiff_app}
    p(\mathbf{x}_{1:T}, \mathbf{z}_{1:T} | \mathbf{u}_{1:T})
    =
     & \underbrace{
        p_{\eta}(\mathbf{z}_{1:j} | \mathbf{u}_{1:j}) p_{\theta}(\mathbf{x}_{1:j} | \mathbf{z}_{1:j}, \mathbf{u}_{1:j})
    }_{\mathclap{\text{State Initialization}}} \\
     & \times
    \prod_{t=j+1}^{T}
    \underbrace{
        p^^t_\psi(\mathbf{z}_t | \mathbf{z}_{t-j:t-1}, \mathbf{u}_{t-j: t })\;
    }_{\mathclap{\text{Transition Density}}}
    \;\times\;
    \underbrace{
        p^^t_\theta(\mathbf{x}_t | \mathbf{z}_{t-j+1:t}, \mathbf{u}_{t-j + 1 : t })
    }_{\mathclap{\text{Emissions Model}}},
    \end{aligned}
\end{equation}
where we have indicated the parameters as explained in the main text.

The inclusion of \( j \) as a hyperparameter for our model is inspired by the GenCast model
  of~\cite{price2025probabilistic}, which found that a second-order Markov assumption (in observation
  space) was more effective for forecasting in contrast to a first-order assumption.
We note that another option for increasing the Markov order would be to use a higher-dimensional
  latent state, where the transition density is still first-order in the latent space, but the latent
  state at time \(t\) includes the previous \(j\) states.
However, the representation here allows for independently controlling the Markov order and the
  dimension of the latent state.
Another motivation for this representation is that it allows us to view our method as a direct
  generalization of autoregressive diffusion models.
For example, the model proposed by~\citet{price2025probabilistic} can be viewed as our model with
  \(j=2\) and identity emissions.

\AFTER{make a note about the relationship beteen \(j\) and the dimension of the latent space \(d\)?
    As an example, an AR(1) model can be extended to an AR(p) by augmenting the state space to include
      the previous \(p\) states.
}

\section{Modeling the Transition Density with Diffusion Models}
To model the transition density \(p(\mathbf{z}_t | \mathbf{z}_{t-j : t-1})\) as a diffusion model,
  we must introduce the intermediate states used in the diffusion process.
To simplify notation, we will denote the history of latent states and covariates as \( \mathbf{c}_t
  := (\mathbf{z}_{t-j:t-1}, \mathbf{u}_{t-j:t})\), as this conditioning set remains fixed across the
  diffusion steps for a given time \(t\).

\cref{eq:diffnot} shows that we may write the transition density as
\begin{equation}
    \label{eq:difftrans}
    p^^t_\psi(\mathbf{z}_t | \mathbf{c}_t)
    =
    \int p^^t_\psi(\mathbf{z}_t^0 | \mathbf{z}_t^1, \mathbf{c}_t) \;
    \prod_{k=2}^K p^^t_\psi(\mathbf{z}_t^{k-1} | \mathbf{z}_t^k, \mathbf{c}_t) \;
    p(\mathbf{z}_t^K) \;
    d\mathbf{z}_t^{1:K},
\end{equation}
where \( \mathbf{z}_t^0 := \mathbf{z}_t \) is the clean latent state, \( p(\mathbf{z}_t^K)\) has standard normal density, and \( p^^t_\psi(\mathbf{z}_t^{k-1} | \mathbf{z}_t^k, \mathbf{z}_{t-j:t-1}) \) is the learned reverse process of the diffusion model.

Thus, we augment \cref{eq:kdiff} to include the intermediate diffusion states \( \mathbf{z}_t^{1:K}
  \) for \( t > j \).

The augmented joint distribution is
\begin{align}
     & p(\mathbf{x}_{1:T}, \mathbf{z}_{1:T}, \left\{\mathbf{z}_t^{1:K}\right\}_{t=j+1}^T | \mathbf{u}_{1:T}) \nonumber \\
     & \quad=p_{\eta,\theta}(\mathbf{z}_{1:j}, \mathbf{x}_{1:j} | \mathbf{u}_{1:j})
    \; \times \;
    \prod_{t=j+1}^{T}
    \Bigg[
        p(\mathbf{z}_t^K) \prod_{k=1}^K p^^t_\psi(\mathbf{z}_t^{k-1} | \mathbf{z}_t^k, \mathbf{c}_t) ) \times
        p^^t_\theta(\mathbf{x}_t | \mathbf{z}_{t-j+1:t}, \mathbf{u}_{t-j + 1 : t })
        \Bigg]. \label{eq:kdiff_augmented}
\end{align}

We may recover \cref{eq:kdiff} by marginalizing out the diffusion states \( \mathbf{z}_t^{1:K} \)
  for \( t > j \).

\section{ELBO Derivation}
\subsection{Posterior Factorization} \label{app:post_fact}
Learning our generative model in \cref{eq:kdiff_augmented} requires performing inference over the
  latent variables \( \mathbf{z}_{1:T} \) and \( \{\mathbf{z}_t^{1:K}\}_{t=j+1}^T \).
To do so, we follow works from the deep state-space modeling literature~\citep{krishnan2015deep,
      krishnan2017structured, girin2022dynamical} and train under the VAE framework which allows jointly
  learning the generative model and the inference model.
The primary objective is to maximize the marginal log-likelihood \( \log p(\mathbf{x}_{1:T} |
  \mathbf{u}_{1:T}) \) with respect to the parameters of the generative model.
As is typical with latent-variable models, exact inference depends on the smoothing posterior
distribution
\begin{equation}
    p(\mathbf{z}_{1:T}, \{\mathbf{z}_t^{1:K}\}_{t=j+1}^T \mid \mathbf{x}_{1:T},
    \mathbf{u}_{1:T}),
\end{equation} whose normalization involves integrating over a high-dimensional latent space
and is generally intractable.

To obtain a tractable objective, we introduce a variational distribution
\begin{equation}
    q_\phi(\mathbf{z}_{1:T}, \left\{\mathbf{z}_t^{1:K}\right\}_{t=j+1}^T | \mathbf{x}_{1:T},
    \mathbf{u}_{1:T})
\end{equation}
to approximate
the true posterior distribution.
It is important to choose a family of variational distributions which matches the conditional
  independence structure of the true posterior distribution~\citep{girin2022dynamical}.
If the variational distribution misses important dependencies required to capture the true
  posterior, then the variational distribution will not be able to approximate the true posterior
  well, resulting in a looser bound on the marginal log-likelihood.

Therefore, we examine the form of the true posterior distribution \( p(\mathbf{z}_{1:T},
  \{\mathbf{z}_t^{1:K}\}_{t=j+1}^T | \mathbf{x}_{1:T}, \mathbf{u}_{1:T}) \) to determine the
  appropriate factorization for the variational distribution.
The conditional independence structure of the true posterior distribution can be determined by
  applying d-separation to the directed graphical model corresponding to the generative model in
  \cref{eq:kdiff_augmented}, resulting in the following lemma.

\begin{lemma} \label{lem:dsep}
    For the generative model in \cref{eq:kdiff_augmented}, the exact smoothing posterior factorizes as
    \begin{align}
        \label{eq:true_post_dsep}
         & p(\mathbf{z}_{1:T}, \{\mathbf{z}_t^{1:K}\}_{t=j+1}^T | \mathbf{x}_{1:T}, \mathbf{u}_{1:T}) \\
         & \quad=
        \left[
            \prod_{t=1}^{j} p(\mathbf{z}_t | \mathbf{x}_{t : T}, \mathbf{z}_{1: t-1}, \mathbf{u}_{1 : t})
            \right] \times
        \left[
            \prod_{t=j+1}^{T} p(\mathbf{z}_t | \mathbf{x}_{t:T}, \mathbf{z}_{t-j: t-1}, \mathbf{u}_{t-j : t})
        \right] \\
         & \qquad\qquad\qquad\times
        \left[
        \prod_{t=j+1}^{T} p(\mathbf{z}_t^{1:K} | \mathbf{z}_t, \mathbf{z}_{t-j: t-1}, \mathbf{u}_{t-j : t})
        \right],
    \end{align}
    where
    \begin{equation}
        p(\mathbf{z}_t^{1:K} | \mathbf{z}_t, \mathbf{z}_{t-j: t-1}, \mathbf{u}_{t-j : t})
        \propto
        p(\mathbf{z}_t^K) \prod_{k=1}^K p^^t_\psi(\mathbf{z}_t^{k-1} | \mathbf{z}_t^k, \mathbf{z}_{t-j:t-1}, \mathbf{u}_{t-j : t }), \label{eq:postfact_diff}
    \end{equation}
    with \( \mathbf{z}_t^0 := \mathbf{z}_t \).
\end{lemma}

\TODO{Does this lemma need some t superscripts?}

\AFTER{Mike: Check Proof.}

\begin{proof}
    Let the reverse-time markov chain \( \mathbf{z}_t^K\to \cdots \to \mathbf{z}_t^1\to \mathbf{z}_t^0
      \) be parameterized by the time-homogeneous transition densities \( p^^t_\psi(\mathbf{z}_t^{k-1} |
      \mathbf{z}_t^k, \mathbf{z}_{t-j:t-1}, \mathbf{u}_{t-j : t }) \).

    Applying Bayes rule to \( p(\mathbf{z}_{1:T}, \{\mathbf{z}_t^{1:K}\}_{t=j+1}^T | \mathbf{x}_{1:T}, \mathbf{u}_{1:T}) \), we have
    \begin{align*}
               & p(\mathbf{z}_{1:T}, \{\mathbf{z}_t^{1:K}\}_{t=j+1}^T | \mathbf{x}_{1:T}, \mathbf{u}_{1:T}) \\
        \qquad & \quad=
        \frac{p(\mathbf{x}_{1:T}, \mathbf{z}_{1:T}, \{\mathbf{z}_t^{1:K}\}_{t=j+1}^T | \mathbf{u}_{1:T})}
        {p(\mathbf{x}_{1:T} | \mathbf{u}_{1:T})}
        \\
        \qquad & \quad\propto
        p_{\eta,\theta}(\mathbf{z}_{1:j}, \mathbf{x}_{1:j} | \mathbf{u}_{1:j}) \\
        \qquad & \qquad \qquad \times \;\prod_{t=j+1}^{T}
        \Bigg[
            p(\mathbf{z}_t^K) \prod_{k=1}^K p^^t_\psi(\mathbf{z}_t^{k-1} | \mathbf{z}_t^k, \mathbf{c}_t) ) \cdot
            p^^t_\theta(\mathbf{x}_t | \mathbf{z}_{t-j+1:t}, \mathbf{u}_{t-j + 1 : t })
            \Bigg].
    \end{align*}
    This induces a directed acyclic graph (DAG) over the variables \( \mathbf{x}_{1:T},
      \mathbf{z}_{1:T}, \{\mathbf{z}_t^{1:K}\}_{t=j+1}^T, \mathbf{u}_{1:T} \).
    \AFTER{ add dag. }

    Applying the chain rule to the desired posterior form, we have
    \begin{align*}
         & p(\mathbf{z}_{1:T}, \{\mathbf{z}_t^{1:K}\}_{t=j+1}^{T} \mid \mathbf{x}_{1:T}, \mathbf{u}_{1:T})
        \;\propto\;
        \prod_{t=1}^{j} p(\mathbf{z}_t \mid \mathbf{x}_{1 : T}, \mathbf{z}_{1:t-1}, \mathbf{u}_{1 : T}) \\
         & \qquad\qquad\cdot
        \prod_{t=j+1}^{T} p(\mathbf{z}_t \mid \mathbf{x}_{1:T}, \mathbf{z}_{1:t-1}, \mathbf{u}_{1:T})\;
        \cdot\;
        \prod_{t=j+1}^{T} p(\mathbf{z}_t^{1:K} \mid \mathbf{x}_{1:T}, \mathbf{u}_{1:T}, \mathbf{z}_{1:T}, \{\mathbf{z}_s^{1:K}\}_{s< t}).
    \end{align*}

    In the augmented DAG, any path from \(\mathbf{z}_t^{1:K}\) to \(\{\mathbf{x}_s\}_{s=1}^{T}\) or to
      latents outside of \( \mathbf{c}_t \) must pass through \( (\mathbf{z}_{t-j : t}) \).
    Hence by d-separation, \[ p(\mathbf{z}_t^{1:K} \mid \mathbf{x}_{1:T}, \mathbf{u}_{1:T},
      \mathbf{z}_{1:T}, \{\mathbf{z}_s^{1:K}\}_{s< t}) \;=\; p(\mathbf{z}_t^{1:K} \mid \mathbf{z}_t,
      \mathbf{c}_t).
    \]
    Moreover, by the Markov structure of the reverse chain, and \(p(\mathbf{z}_t^K)\) being a prior independent of \(\mathbf{c}_t\),
    \[
        p(\mathbf{z}_t^{1:K}\mid \mathbf{z}_t, \mathbf{c}_t)
        \;\propto\;
        p(\mathbf{z}_t^K)\,\prod_{k=1}^{K} p^^t_\psi(\mathbf{z}_t^{k-1}\mid \mathbf{z}_t^{k}, \mathbf{c}_t),
        \qquad \mathbf{z}_t^{0}\equiv \mathbf{z}_t,
    \]
    with the same parameter set \(\psi\) for all \(t\) (time-homogeneous in \(t\)).

    For \( t> j \), by the \(j\)-th order Markov property, \(\mathbf{z}_t \perp\!
    \!\!\perp \mathbf{z}_{1:t-j-1}\mid (\mathbf{z}_{t-j:t-1}, \mathbf{u}_{t-j:t})\).
    Conditioning on all observations, any influence of \(\mathbf{x}_{1:t-1}\) on \(\mathbf{z}_t\) is
      blocked by the parents \(\mathbf{z}_{t-j:t-1}\); thus \[ \mathbf{z}_t \;\perp\!
    \!\!\perp\; \mathbf{x}_{1:t-1}
    \;|\; (\mathbf{z}_{t-j:t-1}, \mathbf{u}_{t-j:t}, \mathbf{x}_{t:T}).
    \]
    Therefore,
    \[
        p(\mathbf{z}_t \mid \mathbf{x}_{1:T}, \mathbf{z}_{1:t-1}, \mathbf{u}_{1:T})
        \;=\;
        p(\mathbf{z}_t \mid \mathbf{x}_{t:T}, \mathbf{z}_{t-j:t-1}, \mathbf{u}_{t-j:t}),
        \qquad t>j.
    \]
    For \( t \leq j \), we similarly have
    \[
        p(\mathbf{z}_t \mid \mathbf{x}_{1:T}, \mathbf{z}_{1:t-1}, \mathbf{u}_{1:T})
        \;=\;
        p(\mathbf{z}_t \mid \mathbf{x}_{t:T}, \mathbf{z}_{1:t-1}, \mathbf{u}_{1 : t}).
    \]

    Substituting the two reductions back into the chain-rule factorization yields
    \begin{align*}
         & p(\mathbf{z}_{1:T}, \{\mathbf{z}_t^{1:K}\}_{t=j+1}^{T} \mid \mathbf{x}_{1:T}, \mathbf{u}_{1:T})
        \;\propto\;
        \Bigg[\prod_{t=1}^{j} p(\mathbf{z}_t \mid \mathbf{x}_{t : T}, \mathbf{z}_{1:t-1}, \mathbf{u}_{1:t})\Bigg]\;
        \cdot\;
        \Bigg[\prod_{t=j+1}^{T} p(\mathbf{z}_t \mid \mathbf{x}_{t:T}, \mathbf{z}_{t-j:t-1}, \mathbf{u}_{t-j:t})\Bigg] \\
         & \hspace{7.3cm}\cdot\;
        \Bigg[\prod_{t=j+1}^{T} p(\mathbf{z}_t^{1:K}\mid \mathbf{z}_t, \mathbf{c}_t)\Bigg].
    \end{align*}
    Re-normalizing both sides gives the stated posterior factorization.
\end{proof}

Once again, if we marginalize the intermediate diffusion states \( \mathbf{z}_t^{1:K} \),
  \cref{eq:postfact_diff} recovers the original transition density \( p^^t_\psi(\mathbf{z}_t |
  \mathbf{z}_{t-j:t-1}, \mathbf{u}_{t-j : t }) \).
This serves to highlight that there are two perspectives to our model.
First, we may view our model as replacing the transition density in a standard state-space model
  with a diffusion model, which allows for more flexible transitions.
Alternatively, we may view our model as a state-space model with an augmented latent space which
  includes the intermediate diffusion states, where marginalization of the intermediate states
  recovers the more-general state-space model in \cref{eq:kdiff}.

The optimal variational posterior factorizes identically to the true
  posterior~\citep{wainwright2008graphical}.
Hence, given \cref{lem:dsep}, we assume the following factorization for the variational posterior density \( q_\phi(\mathbf{z}_{1:T}, \{\mathbf{z}_t^{1:K}\}_{t=j+1}^T | \mathbf{x}_{1:T}, \mathbf{u}_{1:T}) \) without making any approximation other than the dependence on $\phi$:
\begin{align}
    \label{eq:var_post}
     & q_\phi(\mathbf{z}_{1:T},
    \{\mathbf{z}_t^{1:K}\}_{t=j+1}^T | \mathbf{x}_{1:T}, \mathbf{u}_{1:T})
    =
    \underbrace{
    q_\phi\!\left(\mathbf{z}_{1 : j} | \mathbf{x}_{1 : T}, \mathbf{u}_{1 : j}\right) }_{\mathclap{\text{State Initialization}}} \\
     & \;\times\;
    \underbrace{\prod_{t=j+1}^{T}
        q^^t_\phi\!\left(\mathbf{z}_t |\, \mathbf{z}_{t-j:t-1},\mathbf{u}_{t-j  : t}, \mathbf{x}_{t:T}\right)
    }_{\mathclap{\text{Smoothed Transitions}}}
    \;\times\;
    \underbrace{\prod_{t=j+1}^{T}
    q_{\text{chain}}(\mathbf{z}_t^{1:K} | \mathbf{z}_t)
    }_{\mathclap{\text{Diffusion Noising Process}}}.
\end{align}

Note the dependence of the true and variational posteriors of the latent state \( \mathbf{z}_t \)
  on the future observations \( \mathbf{x}_{t:T} \); the same structure is used in the smoothing
  posterior of the DKF~\citep{krishnan2017structured} and other works on deep state-space
  models~\citep{girin2022dynamical}.
To implement this dependence, we follow~\citet{krishnan2017structured} and employ a neural-network
  to summarize the future observations \( \mathbf{x}_{t:T} \) for each timestep \(t\) into a
  fixed-dimensional vector \(\mathbf{h}_t\).
We define
\begin{equation}
    \label{eq:future_summary_def}
    \left(\mathbf{h}_T, \mathbf{h}_{T-1}, \ldots \mathbf{h}_1\right) = F_\phi(\text{concat}(\mathbf{x}_{1 : T}))
\end{equation}
The implementation of \(F_\phi\) used by~\citet{krishnan2017structured} is a backward RNN, which
  produces a sequence of hidden states \( \mathbf{h}_T, \mathbf{h}_{T-1}, \ldots, \mathbf{h}_1 \).
By design each \(\mathbf{h}_t\) is intended to summarize the future observations for each time
  step.
We employ a similar implementation, for which details can be found in
  \cref{app:encoder_architecture}.

\subsection{ELBO} \label{app:elbo_full_derive}
We now derive an evidence lower bound (ELBO) for our model, by finding a lower bound on the
  marginal log-likelihood \( \log p(\mathbf{x}_{1:T} | \mathbf{u}_{1 : T}) \).
We have
\begin{align*}
    \log p\left(\mathbf{x}_{1:T} | \mathbf{u}_{1 : T}\right)
     & = \log \int p\left(\mathbf{x}_{1:T}, \mathbf{z}_{1:T}, \left\{\mathbf{z}_t^{1:K}\right\}_{t=j+1}^T | \mathbf{u}_{1 : T}\right) d\mathbf{Z} \\
     & = \log \int q_\phi\left( \mathbf{z}_{1:T}, \left\{\mathbf{z}_t^{1:K}\right\}_{t=j+1}^T| \mathbf{x}_{1:T}, u_{1 : T}\right) \frac{ p\left(\mathbf{x}_{1:T}, \mathbf{z}_{1:T}, \left\{\mathbf{z}_t^{1:K}\right\}_{t=j+1}^T | \mathbf{u}_{1 : T}\right)}{q_\phi\left( \mathbf{z}_{1:T}, \left\{\mathbf{z}_t^{1:K}\right\}_{t=j+1}^T| \mathbf{x}_{1:T}, u_{1 : T}\right)} d\mathbf{Z} \\
     & \geq \mathbb{E}_{q_\phi\left( \mathbf{z}_{1:T}, \left\{\mathbf{z}_t^{1:K}\right\}_{t=j+1}^T| \mathbf{x}_{1:T}, \mathbf{u}_{1 : T}\right)} \Bigg[ \\
     & \qquad\quad\log p\left(\mathbf{x}_{1:T},
    \mathbf{z}_{1:T}, \left\{\mathbf{z}_t^{1:K}\right\}_{t=j+1}^T | \mathbf{u}_{1 : T}\right) \\
     & \qquad-\log  q_\phi\left( \mathbf{z}_{1:T}, \left\{\mathbf{z}_t^{1:K}\right\}_{t=j+1}^T| \mathbf{x}_{1:T}, \mathbf{u}_{1 : T}\right)\Bigg], \\
\end{align*}
where for shorthand we have written \( d\mathbf{Z} = d\mathbf{z}_{1:T} \; d\left\{\mathbf{z}_t^{1:K}\right\}_{t=j+1}^T \).
Substituting in \cref{eq:kdiff} and \cref{eq:var_post}, we have
\begin{align*}
     & \log p\left(\mathbf{x}_{1:T} | \mathbf{u}_{1 : T}\right) \\
     & \geq \mathbb{E}_{q_\phi} \Big[ \\
     & \qquad\quad\log p_{\eta,\theta}(\mathbf{z}_{1:j}, \mathbf{x}_{1:j} | \mathbf{u}_{1:j}) \\
     & \qquad+ \sum_{t=j+1}^T \log p^^t_\theta\left(\mathbf{x}_t | \mathbf{z}_{t-j+1 : t}, \mathbf{u}_{t-j + 1 : t})\right) \\
     & + \sum_{t=j+1}^T \log p(\mathbf{z}_t^K) + \sum_{t=j+1} \sum_{k=1}^K \log p^^t_\psi(\mathbf{z}_t^{k-1} | \mathbf{z}_t^k, \mathbf{c}_t) ) \\
     & \qquad- \sum_{t=1}^j \log q_\phi\left(\mathbf{z}_{1 : j} | \mathbf{x}_{1 : T}, \mathbf{u}_{1 : t}\right) \\
     & \qquad- \sum_{t=j+1}^T \log q^^t_\phi\left(\mathbf{z}_t | \mathbf{x}_{t:T}, \mathbf{z}_{t-j : t-1}, \mathbf{u}_{t-j : t}\right) \\
     & \qquad- \sum_{t=j+1}^T \sum_{k=1}^K \log q(\mathbf{z}_t^{k} | \mathbf{z}_t^{k-1}) \Big]. \\
\end{align*}
Grouping terms,
\begin{align*}
     & \log p\left(\mathbf{x}_{1:T} | \mathbf{u}_{1 : T}\right) \nonumber \\
     & \geq \mathbb{E}_{q_\phi} \Big[ \nonumber \\
     & \qquad\quad\log p_{\eta,\theta}(\mathbf{z}_{1:j}, \mathbf{x}_{1:j} | \mathbf{u}_{1:j}) \nonumber \\
     & \qquad- \log q_\phi\left(\mathbf{z}_{1 : j} | \mathbf{x}_{1 : T}, \mathbf{u}_{1 : t}\right) \nonumber \\
     & \qquad+ \sum_{t=j+1}^T \log p^^t_\theta\left(\mathbf{x}_t | \mathbf{z}_{t-j+1 : t}, \mathbf{u}_{t-j + 1 : t})\right) \nonumber \\
     & + \sum_{t=j+1}^T \Big[\log p(\mathbf{z}_t^K) \nonumber \\
     & \qquad+ \sum_{k=1}^K \log p^^t_\psi(\mathbf{z}_t^{k-1} | \mathbf{z}_t^k, \mathbf{c}_t) - \log q(\mathbf{z}_t^{k} | \mathbf{z}_t^{k-1}) \Big]. \nonumber \\
     & \qquad- \sum_{t=j+1}^T \log q^^t_\phi\left(\mathbf{z}_t | \mathbf{x}_{t:T}, \mathbf{z}_{t-j : t-1}, \mathbf{u}_{t-j : t}\right) \nonumber \\
     & = \mathbb{E}_{q_\phi(\mathbf{z}_{1 : T} | \mathbf{x}_{1: T}, \mathbf{u}_{1: T})}
    \left[\log p_{\eta,\theta}(\mathbf{z}_{1:j}, \mathbf{x}_{1:j} | \mathbf{u}_{1:j}) - \log q_\phi\left(\mathbf{z}_{1 : j} | \mathbf{x}_{1 : T}, \mathbf{u}_{1 : t}\right)\right] \\
     & + \sum_{t=j+1}^T \mathbb{E}_{q_\phi(\mathbf{z}_{1:T} | \mathbf{x}_{1:T}, \mathbf{u}_{1 : T})} \Big[ \log p^^t_\theta\left(\mathbf{x}_t | \mathbf{z}_{t-j+1 : t}, \mathbf{u}_{t-j + 1 : t}\right)\Big] \\
     & + \sum_{t=j+1}^T \mathbb{E}_{q_\phi(\mathbf{z}_{1:T}, \left\{\mathbf{z}_t^{1:K}\right\}_{t=j+1}^T | \mathbf{x}_{1:T}, \mathbf{u}_{1 : T})} \Big[\log p(\mathbf{z}_t^K) + \sum_{k=1}^K \log p^^t_\psi(\mathbf{z}_t^{k-1} | \mathbf{z}_t^k, \mathbf{c}_t) - \log q(\mathbf{z}_t^{k} | \mathbf{z}_t^{k-1}) \Big] \\
     & - \sum_{t=j+1}^T \mathbb{E}_{q_\phi(\mathbf{z}_{1:T} | \mathbf{x}_{1:T}, \mathbf{u}_{1 : T})} \Big[\log q^^t_\phi\left(\mathbf{z}_t | \mathbf{x}_{t:T}, \mathbf{z}_{t-j : t-1}, \mathbf{u}_{t-j : t}\right)\Big]                                                                                             . \\
\end{align*}
Expanding \(\log p_{\eta,\theta}(\mathbf{z}_{1:j}, \mathbf{x}_{1:j} | \mathbf{u}_{1:j}) = \log p_\eta(\mathbf{z}_{1:j} | \mathbf{u}_{1 : j}) + \log p^^t_\theta( \mathbf{x}_{1:j} | \mathbf{z}_{1 : j},  \mathbf{u}_{1:j})\),
\begin{align*}
     & \log p\left(\mathbf{x}_{1:T} | \mathbf{u}_{1 : T}\right) \nonumber \\
     & \geq \mathbb{E}_{q_\phi(\mathbf{z}_{1 : T} | \mathbf{x}_{1: T}, \mathbf{u}_{1: T})}
    \left[\log p_\eta(\mathbf{z}_{1:j} | \mathbf{u}_{1 : j})  - \log q_\phi\left(\mathbf{z}_{1 : j} | \mathbf{x}_{1 : T}, \mathbf{u}_{1 : t}\right)\right] \\
     & + \sum_{t=1}^T \mathbb{E}_{q_\phi(\mathbf{z}_{1:T} | \mathbf{x}_{1:T}, \mathbf{u}_{1 : T})} \Big[ \log p^^t_\theta\left(\mathbf{x}_t | \mathbf{z}_{\max(t-j+1,1) : t}, \mathbf{u}_{\max(t-j + 1, 1) : t}\right)\Big] \\
     & + \sum_{t=j+1}^T \mathbb{E}_{q_\phi(\mathbf{z}_{1:T}, \left\{\mathbf{z}_t^{1:K}\right\}_{t=j+1}^T | \mathbf{x}_{1:T}, \mathbf{u}_{1 : T})} \Big[\log p(\mathbf{z}_t^K) + \sum_{k=1}^K \log p^^t_\psi(\mathbf{z}_t^{k-1} | \mathbf{z}_t^k, \mathbf{c}_t) - \log q(\mathbf{z}_t^{k} | \mathbf{z}_t^{k-1}) \Big] \\
     & - \sum_{t=j+1}^T \mathbb{E}_{q_\phi(\mathbf{z}_{1:T} | \mathbf{x}_{1:T}, \mathbf{u}_{1 : T})} \Big[\log q^^t_\phi\left(\mathbf{z}_t | \mathbf{x}_{t:T}, \mathbf{z}_{t-j : t-1}, \mathbf{u}_{t-j : t}\right)\Big] . \\
\end{align*}
To obtain a minimization objective we multiply by \(-1\) yielding the negative ELBO,

\begin{align}
     & -\log p\left(\mathbf{x}_{1:T} | \mathbf{u}_{1 : T}\right) \nonumber \\
     & \leq  \underbrace{\mathbb{E}_{q_\phi(\mathbf{z}_{1 : T} | \mathbf{x}_{1: T}, \mathbf{u}_{1: T})}
    \left[\log q_\phi\left(\mathbf{z}_{1 : j} | \mathbf{x}_{1 : T}, \mathbf{u}_{1 : t}\right) - \log p_{\eta}(\mathbf{z}_{1:j} | \mathbf{u}_{1:j})\right]}_{\text{Initialization Loss}} \label{eq:term_init_loss} \\
     & + \underbrace{\sum_{t=1}^T \mathbb{E}_{q_\phi(\mathbf{z}_{1:T} | \mathbf{x}_{1:T}, \mathbf{u}_{1 : T})} \Big[ - \log p^^t_\theta\left(\mathbf{x}_t | \mathbf{z}_{\max(t-j+1,1) : t}, \mathbf{u}_{\max(t-j + 1, 1) : t}\right)\Big]}_{\text{Reconstruction Loss}} \label{eq:term_seq_reconstruct} \\
     & + \underbrace{\sum_{t=j+1}^T \mathbb{E}_{q_\phi(\mathbf{z}_{1:T}, \left\{\mathbf{z}_t^{1:K}\right\}_{t=j+1}^T | \mathbf{x}_{1:T}, \mathbf{u}_{1 : T})} \Big[- \log p(\mathbf{z}_t^K) + \sum_{k=1}^K \log q(\mathbf{z}_t^{k} | \mathbf{z}_t^{k-1}) -\log p^^t_\psi(\mathbf{z}_t^{k-1} | \mathbf{z}_t^k, \mathbf{c}_t)\Big]}_{\text{Diffusion Chain}} \label{eq:term_diff_chain} \\
     & + \underbrace{\sum_{t=j+1}^T \mathbb{E}_{q_\phi(\mathbf{z}_{1:T} | \mathbf{x}_{1:T}, \mathbf{u}_{1 : T})} \Big[\log q^^t_\phi\left(\mathbf{z}_t | \mathbf{x}_{t:T}, \mathbf{z}_{t-j : t-1}, \mathbf{u}_{t-j : t}\right)\Big]}_{\text{Negative of Posterior Entropy}}. \label{eq:post_entropy}.
\end{align}

For convenience, we will refer to each of the four terms in
  \cref{eq:term_init_loss,eq:term_seq_reconstruct,eq:term_diff_chain,eq:post_entropy} as \(
  \mathcal{L}_{\text{init}}(\phi,\eta,\theta; \mathbf{x}_{1 : T}, \mathbf{u}_{1 : T}) \), \(
  \mathcal{L}_{\text{recon}}(\phi,\theta; \mathbf{x}_{1 : T}, \mathbf{u}_{1 : T}) \), \(
  \mathcal{L}_{\text{diff}}(\phi,\psi; \mathbf{x}_{1 : T}, \mathbf{u}_{1 : T}) \), and \(
  \mathcal{L}_{\text{entropy}}(\phi; \mathbf{x}_{1 : T}, \mathbf{u}_{1 : T}) \), respectively.

Yet again we show the marginalized perspective, where marginalization of the intermediate diffusion
  states \( \mathbf{z}_t^{1:K} \) in \cref{eq:term_diff_chain} recovers the transition density \(
  p^^t_\psi(\mathbf{z}_t | \mathbf{z}_{t-j : t-1}) \).
Combining \cref{eq:term_diff_chain} with \cref{eq:post_entropy} post-marginalization, we recover
the KL divergence between the smoothed transition posterior and the transition density,
\begin{equation}
    \label{eq:marginalized_kl}
    \KL{q^^t_\phi(\mathbf{z}_t | \mathbf{x}_{t:T}, \mathbf{z}_{t-j : t-1})}
    {p^^t_\psi(\mathbf{z}_t | \mathbf{z}_{t-j : t-1})},
\end{equation}
again making a \(j\)-th order extension of the form given by~\cite{krishnan2015deep, krishnan2017structured}.


\section{Diffusion Modeling}
\subsection{Background for Conditional Diffusion models} \label{app:diffusion_background}

Diffusion models are a class of generative models which learn to sample from some arbitrary data
  distribution by learning to reverse a gradual noising process.
In the conditional setting, we wish to learn a model of the form \( p(\mathbf{x} | \mathbf{c}) \),
  where \( \mathbf{c} \) is some conditioning information.
To formulate the diffusion objective, we first must define a forward noising process.
This process may be defined in either continuous or discrete time; we present the discrete-time
  case here.
In discrete-time, the forward noising process is a Markov chain which gradually adds Gaussian noise
  to the data.
As shown by~\cite{ho2020denoising}, the forward noising process first takes an initial data point
  \( \mathbf{z}^0 := \mathbf{z} \) sampled from the data, and then iteratively adds noise to produce
  a sequence of noisy latent states \( \mathbf{z}^1, \mathbf{z}^2, \ldots, \mathbf{z}^K \).
The corruption process for each step \( k \) is defined by
a transition density
\begin{equation*}
    q(\mathbf{z}^k | \mathbf{z}^{k-1}) = \mathcal{N}(\mathbf{z}^k; \sqrt{1 - \beta_k}\mathbf{z}^{k-1}, \beta_k\mathbf{I}),
\end{equation*}
where \( 0 < \beta_1, \beta_2, \ldots, \beta_K < 1 \) is a fixed noise schedule,

The noise schedule is chosen such that the terminal latent state \( \mathbf{z}^K \) is
  approximately distributed as \( \mathcal{N}(0,\mathbf{I}) \).

This forward process allows direct sampling of a latent \( \mathbf{z}_t^k \) for some time step \(
  t \) since \( q(\mathbf{z}_t^k| \mathbf{z}_t^{0}) \sim \mathcal{N}(\mathbf{z}_t^k ;
  \sqrt{\overline{\alpha}_k}\mathbf{z}_t^0, (1 - \overline{\alpha}_k)\mathbf{I})\), where \( \alpha_k
  = (1 - \beta_k) \) and \( \overline{\alpha}_k = \prod_{i=1}^k \alpha_i \).

For data \( \mathbf{z} \), a reversal process may be defined as
\begin{equation*}
    p^^t_\psi \left(\mathbf{z}^{k-1} | \mathbf{z}^{k}, \mathbf{c}\right),
\end{equation*}
such that
\begin{align}
    p^^t_\psi(\mathbf{z}^0 | \mathbf{c}) & = \int p^^t_\psi \left(\mathbf{z}^{0:K-1} | \mathbf{z}^{K}, \mathbf{c}\right) p(\mathbf{z}^K) d\mathbf{z}^{1:K} \nonumber \\
                                         & = \int \prod_{k=K}^1 p^^t_\psi \left(\mathbf{z}^{k-1} | \mathbf{z}^{k}, \mathbf{c}\right) p(\mathbf{z}^K) d\mathbf{z}^{1:K} \label{eq:diffnot}
\end{align}
where \(p^^t_\psi(\mathbf{z}^{k-1} | \mathbf{z}^k, \mathbf{c}) \) is learned by a neural network with parameters \( \psi \).
Typically, a single neural network with shared parameters \( \psi \) is used to learn the reverse
  process across all noise levels, with the noise level \( k \) fed in as an additional input to the
  network.

Training this model can be accomplished by minimizing the variational bound on the negative
  log-likelihood of the data, we provide details in \cref{app:elbo_full_derive} as this is key to our
  methodology.
Sampling from the diffusion model is accomplished by first sampling \( \mathbf{z}^K \sim
  \mathcal{N}(0,\mathbf{I}) \), and then sampling \( \mathbf{z}^{K-1}, \mathbf{z}^{K-2}, \ldots,
  \mathbf{z}^0 \) using the learned reverse process, however more efficient samplers are used in
  practice.
We refer the reader to~\citet{ho2020denoising, song2020score, karras2022elucidating} for more
  details on diffusion models.

\subsection{Diffusion Reparameterization} \label{app:diff_reparameterization}
We follow prior works and reparameterize the diffusion model in terms of a noise prediction network \( F_\psi(\cdot) \),
which predicts the noise added at each step of the diffusion process,
so that the diffusion loss may be written as
\begin{align}
    \mathcal{L}_{\widehat{\text{diff}}}(\phi,\psi; \mathbf{x}_{1 : T}, \mathbf{u}_{1 : T})
     & = \sum_{t=j+1}^T \mathbb{E}_{q_\phi(\mathbf{z}_{1:T} | \mathbf{x}_{1 : T}, \mathbf{u}_{1 : T})} \Big[
    \KL{q(\mathbf{z}_t^K | \mathbf{z}_t^0)}{p(\mathbf{z}_t^K)} \nonumber \\
     & \qquad+ \mathbb{E}_{\epsilon \sim \mathcal{N}(0,\mathbf{I}), k \sim \text{Unif}\{1, \ldots, K\}}
        \left[
            K\tilde{w}_k ||F_\psi(\mathbf{z}_t^k, \mathbf{c}_t, k) - \mathbf{y}(\mathbf{z}_t^0,k)||_2^2
            \right]
        \Big]. \label{eq:diff_loss}
\end{align}

The derivation and motivation for \cref{eq:diff_loss} is given in
  \cref{app:faithful_diff_elbo_proof}.
In the above equation, \( F_\psi \) denotes the neural network parameterizing the reverse diffusion
  process, conditioned on the noisy latent \( \mathbf{z}_t^k \), the Markov history \( \mathbf{c}_t
  \), and the diffusion step \( k \).
The target \( \mathbf{y}(\mathbf{z}_t^0, k) \) is a function of the clean latent \( \mathbf{z}_t^0
  \) and the preconditioned noisy latent, as defined in \cref{app:faithful_diff_elbo_proof} (see
  \cref{eq:diff_target_precond}).
The coefficient \( \tilde{w}_k \) weights the loss at each noise level and is given by
  \cref{eq:is-weight} and its subsequent scaling.

We have written the diffusion loss as \( \mathcal{L}_{\widehat{\text{diff}}} \) to distinguish it
  from the KL form in \cref{eq:marginalized_kl}.

\section{Derivation of the diffusion loss} \label{app:faithful_diff_elbo_proof}
\subsection{Derivation of Noise-Prediction Loss}

We now proceed with describing our specific loss for the diffusion transitions.
Notice that the diffusion model is trained under an expectation over the variational distribution \(q_\phi \).
This means that the loss of the diffusion model is informative to the encoder, so we must carefully
  consider the full ELBO objective during our derivation.

With the goal of maximizing the log-likelihood of \(p^^t_\psi(\mathbf{z}_t^0 | \mathbf{c}_t)\), we
  can derive a variational bound on the negative log-likelihood of the transition density.
We have
\begin{subequations}
    \begin{align}
         & -\log p^^t_\psi(\mathbf{z}_t^0 | \mathbf{c}_t) \nonumber \\
         & = - \log \int p^^t_\psi(\mathbf{z}_t^{0:K-1} | \mathbf{z}_t^K, \mathbf{c}_t)p(\mathbf{z}_t^K) d\mathbf{z}^{1:K}_t \label{eq:cond_derive_1} \\
         & = - \log \int q(\mathbf{z}_t^{1:K} | \mathbf{z}_t^0) \frac{p^^t_\psi(\mathbf{z}_t^{0:K-1} | \mathbf{z}_t^K, \mathbf{c}_t)p(\mathbf{z}_t^K)}{q(\mathbf{z}_t^{1:K} | \mathbf{z}_t^0)} d\mathbf{z}^{1:K}_t \\
         & \leq \mathbb{E}_{q(\mathbf{z}_t^{1:K} | \mathbf{z}_t^0)} \left[ - \log \frac{p^^t_\psi(\mathbf{z}_t^{0:K-1} | \mathbf{z}_t^K, \mathbf{c}_t)p(\mathbf{z}_t^K)}{q(\mathbf{z}_t^{1:K} | \mathbf{z}_t^0)} \right] \\
         & = \mathbb{E}_{q(\mathbf{z}_t^{1:K}|\mathbf{z}^0_t)}\left[
            -\log p(\mathbf{z}_t^K) - \sum_{k=1}^K \log \frac{p^^t_\psi(\mathbf{z}_t^{k-1} | \mathbf{z}_t^{k}, \mathbf{c}_t)}{q(\mathbf{z}_t^{k} | \mathbf{z}_t^{k-1}, \mathbf{z}_t^0)}
        \right] \\
         & = \mathbb{E}_{q(\mathbf{z}_t^{1:K}|\mathbf{z}^0_t)}\left[
            -\log p(\mathbf{z}_t^K) - \sum_{k=2}^K \log \frac{p^^t_\psi(\mathbf{z}_t^{k-1} | \mathbf{z}_t^{k}, \mathbf{c}_t)}{q(\mathbf{z}_t^{k} | \mathbf{z}_t^{k-1}, \mathbf{z}_t^0)} - \log \frac{p^^t_\psi(\mathbf{z}_t^{0} | \mathbf{z}_t^{1}, \mathbf{c}_t)}{q(\mathbf{z}_t^{1} | \mathbf{z}_t^{0})}
        \right] \\
         & = \mathbb{E}_{q(\mathbf{z}_t^{1:K}|\mathbf{z}^0_t)}\Big[
        -\log p(\mathbf{z}_t^K) - \sum_{k=2}^K \log \frac{p^^t_\psi(\mathbf{z}_t^{k-1} | \mathbf{z}_t^{k}, \mathbf{c}_t)}{q(\mathbf{z}_t^{k-1} | \mathbf{z}_t^{k}, \mathbf{z}_t^0)} \cdot \frac{q(\mathbf{z}_t^{k-1} | \mathbf{z}_t^0)}{q(\mathbf{z}_t^k | \mathbf{z}_t^0)} \nonumber \\
         & \qquad\qquad\qquad- \log \frac{p^^t_\psi(\mathbf{z}_t^{0} | \mathbf{z}_t^{1}, \mathbf{c}_t)}{q(\mathbf{z}_t^{1} | \mathbf{z}_t^{0})}
        \Big] \\
         & = \mathbb{E}_{q(\mathbf{z}_t^{1:K}|\mathbf{z}^0_t)}\Big[
        -\log \frac{p(\mathbf{z}_t^K)}{q(\mathbf{z}_t^K | \mathbf{z}_t^0)} - \sum_{k=2}^K \log \frac{p^^t_\psi(\mathbf{z}_t^{k-1} | \mathbf{z}_t^{k}, \mathbf{c}_t)}{q(\mathbf{z}_t^{k-1} | \mathbf{z}_t^{k}, \mathbf{z}_t^0)} - \log p^^t_\psi(\mathbf{z}_t^0 | \mathbf{z}_t^1, \mathbf{c}_t)   \Big] \\
         & = \KL{q(\mathbf{z}_t^K | \mathbf{z}_t^0)}{p(\mathbf{z}_t^K)} +
        \sum_{k=2}^K \mathbb{E}_{q(\mathbf{z}_t^k | \mathbf{z}_t^0)} \left[ \KL{q(\mathbf{z}_t^{k-1}|\mathbf{z}_t^k, \mathbf{z}^0_{t})}{p^^t_\psi(\mathbf{z}_t^{k-1} | \mathbf{z}_t^{k}, \mathbf{c}_t)} \right] \nonumber \\
         & \qquad\qquad\qquad-\mathbb{E}_{q(\mathbf{z}_t^{1}|\mathbf{z}^0_t)}\left[ \log p^^t_\psi(\mathbf{z}_t^0 | \mathbf{z}_t^1, \mathbf{c}_t) \right].\label{eq:cond_diff_elbo}
    \end{align}
\end{subequations}
This derivation is very similar to the one presented by~\citet{ho2020denoising}, except for the
  presence of conditioning information.
The only additional step needed is made on the first line, where we utilize that the forward
  process is designed so that \( \mathbf{z}_t^K \sim \mathcal{N}(\mathbf{0},\mathbf{I}) \).
This implies that \( p(\mathbf{z}_t^K | \mathbf{z}_{t-j : t-1}) = p(\mathbf{z}_t^K) \), allowing
  the simplification seen after applying the chain rule in \cref{eq:cond_derive_1}.
Also note that the use of conditioning information here is exactly the same in function as
  presented by~\citet{tashiro2021csdi}, and the included conditioning information is not limited to
  the previous latent states.

\Cref{eq:cond_diff_elbo} reveals that training the diffusion model is accomplished by
simply teaching the diffusion model to learn the reverse process by matching
the reverse transition density to the forward transition density at each noise scale.

The first term in \cref{eq:cond_diff_elbo} is constant with respect to \( \psi \), and solely
  depends on the noise schedule and the initial latent state \( \mathbf{z}_t^0 \).
Most works simply drop this term, since \( \mathbf{z}_t^0 \) in their case is usually data, and not
  trainable.
In our case \( \mathbf{z}_t^0 \) is sampled from the encoder, and therefore we must keep this first
  term to respect the true ELBO.

We denote this term as \( \mathcal{L}_{\text{forward}}(\mathbf{z}_t^0) \), and compute it in closed form as
\begin{align*}
    \label{eq:diff_term1}
    \KL{q(\mathbf{z}_t^K | \mathbf{z}_t^0)}{p(\mathbf{z}_t^K)}
     & = \KL{                                                                                                                              \mathcal{N}(\sqrt{\overline{\alpha}_K}\mathbf{z}_t^0, (1 - \overline{\alpha}_K)\mathbf{I})}{\mathcal{N}(0,\mathbf{I})} \\
     & = \frac{1}{2}\left[d(1 - \overline{\alpha}_K) + \overline{\alpha}_K||\mathbf{z}_t^0||^2 - d - d\log(1 - \overline{\alpha}_K)\right] \\
     & = \frac{1}{2}\left[\overline{\alpha}_K||\mathbf{z}_t^0||^2 - d\overline{\alpha}_K - d\log(1 - \overline{\alpha}_K)\right] \\
     & :=                                                                                                                                  \mathcal{L}_{\text{forward}}(\mathbf{z}_t^0).
\end{align*}
For the second term,~\citet{ho2020denoising} uses Bayes rule to derive the
parameterization of the fixed reverse direction kernel as
\begin{align*}
    q(\mathbf{z}_t^{k-1} | \mathbf{z}_t^k, \mathbf{z}_{t}^0)     & = \mathcal{N}(\mathbf{z}_t^{k-1}; \tilde{\mu}_k(\mathbf{z}_t^k, \mathbf{z}_{t}^0), \tilde{\beta}_k\mathbf{I}), \\
    \text{where }\tilde{\mu}_k(\mathbf{z}_t^k, \mathbf{z}_{t}^0) & = \frac{\sqrt{\overline{\alpha}_{k-1}}\beta_k}{1 - \overline{\alpha}_k}\mathbf{z}_t^0 + \frac{\sqrt{\alpha_k}(1-\overline{\alpha}_{k-1})}{1 - \overline{\alpha}_k}\mathbf{z}_t^k \text{ and } \tilde{\beta}_k = \frac{1 - \overline{\alpha}_{k-1}}{1 - \overline{\alpha}_k}\beta_k.
\end{align*}

To turn the training objective of the diffusion model into a regression problem, we must
  parameterize the reverse direction kernel \( p^^t_\psi(\mathbf{z}_t^{k-1} | \mathbf{z}_t^k,
  \mathbf{c}_t) \) in a way that allows us to reparameterize the KL divergence as a noise prediction
  error.
We continue to follow~\citet{ho2020denoising}, and parameterize the reverse direction kernel as \(
  p^^t_\psi(\mathbf{z}_t^{k-1} | \mathbf{z}_t^k, \mathbf{c}_t) = \mathcal{N}(\mathbf{z}_{t}^{k-1} ;
  \mu_\psi(\mathbf{z}_t^k, k, \mathbf{c}_t), \sigma^2\mathbf{I}) \).
Then, we may reparameterize the predicted mean \( \mu_\psi \) as
\begin{equation}
    \label{eq:diff_mean_reparam}
    \mu_\psi(\mathbf{z}_t^k, k, \mathbf{c}_t) = \frac{1}{\sqrt{\alpha_k}}\left(\mathbf{z}_t^k - \frac{\beta_k}{\sqrt{1 - \overline{\alpha}_k}}\epsilon_\psi(\mathbf{z}_t^k, k, \mathbf{c}_t)\right),
\end{equation}
where \( \epsilon_\psi \) is a neural network with parameters \( \psi \) that predicts the noise added at step \( k \).

Under this reparameterization,~\citet{ho2020denoising} shows that the second KL term of \cref{eq:cond_diff_elbo} for a fixed \( k \) is equivalent to
\begin{equation}
    \label{eq:diff_term2}
    \mathbb{E}_{\mathbf{z}_t^0 \sim q, \epsilon \sim \mathcal{N}(\mathbf{0},\mathbf{I})}\left[\frac{\beta_k^2}{2 \sigma^2_k\alpha_k(1- \overline{\alpha}_k)} \left\| \epsilon - \epsilon_\psi\left(\sqrt{\overline{\alpha}_k}\mathbf{z}_t^0 + \sqrt{1 - \overline{\alpha}_k}\epsilon, k, \mathbf{c}_t\right) \right\|^2_2 \right] + C_k,
\end{equation}
where \( C_k \) is solely a function of the noise schedule.

For the last term of \cref{eq:cond_diff_elbo}, observe that
when \( k=1 \), that
\begin{align*}
    \KL{q(\mathbf{z}_t^{k-1}|\mathbf{z}_t^k, \mathbf{z}^0_{t})}{p^^t_\psi(\mathbf{z}_t^{k-1} | \mathbf{z}_t^{k}, \mathbf{c}_t)}
     & =\KL{q(\mathbf{z}_t^{0}|\mathbf{z}_t^1, \mathbf{z}^0_{t})}{p^^t_\psi(\mathbf{z}_t^{0} | \mathbf{z}_t^{1}, \mathbf{c}_t)}.
\end{align*}
The term \( q(\mathbf{z}_t^{0}|\mathbf{z}_t^1, \mathbf{z}^0_{t}) \) is a Dirac delta, i.e., it has zero entropy and
\begin{align*}
    \KL{q(\mathbf{z}_t^{0}|\mathbf{z}_t^1, \mathbf{z}^0_{t})}{p^^t_\psi(\mathbf{z}_t^{0} | \mathbf{z}_t^{1}, \mathbf{c}_t)}
     & = -\mathbb{E}_{q(\mathbf{z}_t^{1}|\mathbf{z}^0_{t})}\left[\log p^^t_\psi(\mathbf{z}_t^{0} | \mathbf{z}_t^{1}, \mathbf{c}_t)\right].
\end{align*}
It follows that the last term in \cref{eq:cond_diff_elbo} is equivalent to the second term of
  \cref{eq:cond_diff_elbo} with \( k=1 \), allowing us to combine these two terms into a single sum.

Rewriting \cref{eq:cond_diff_elbo} with the above results, we have
\begin{equation}
    \label{eq:diff_elbo}
    \begin{aligned}
     & \mathbb{E}_{\mathbf{z}_t \sim q^^t_\phi}[-\log p^^t_\psi(\mathbf{z}_t^0 | \mathbf{c}_t)] \\
     & \leq \mathbb{E}_{\mathbf{z}_t^0 \sim q^^t_\phi}\Bigg[\mathcal{L}_{\text{forward}}(\mathbf{z}_t^0)
        + \underbrace{\sum_{k=1}^K \mathbb{E}_{ \epsilon \sim \mathcal{N}(\mathbf{0},\mathbf{I})}
            \left[w_k \left\|                      \epsilon - \epsilon_\psi\left(\tilde{\mathbf{z}}_t^k, k, \mathbf{c}_t\right) \right\|^2_2 \right]}_{\mathclap{:= \mathcal{L}_{\text{diff}}}} + C\Bigg].
    \end{aligned}
\end{equation}
where
\( w_k = \frac{\beta_k}{2 \alpha_k(1- \overline{\alpha}_k)} \),
\( \tilde{\mathbf{z}}_t^k = \sqrt{\overline{\alpha}_k}\mathbf{z}_t^0 + \sqrt{1 - \overline{\alpha}_k}\epsilon\),
and \(C\) is summation of the \( C_k \) terms.

As \(C\) still is solely a function of the noise schedule, we can drop it from the optimization
  objective.
\subsection{Preconditioning as a Variational Bound}
While training standard diffusion models, the signal-to-noise ratio of the noisy inputs to the
  diffusion model changes drastically across different noise levels.
This results in the magnitude of the regression targets for the noise prediction model to
  fluctuate, and thus the magnitude of the gradients of the diffusion model to fluctuate drastically
  as well.
As the diffusion model is being trained alongside the encoder and decoder, this results in high
  variance gradients for the encoder and decoder, which may lead to instability in training.

To sidestep this issue and ensure stability of the gradients, we employ the preconditioning
  introduced by~\citet{karras2022elucidating}, where the proposed preconditioning re-scales the
  inputs and outputs of the neural network targets.
This rescaling enforces that magnitudes are roughly unit variance across all noise levels, and thus
  the gradients of the diffusion model are more stable across noise levels.

The purpose of this section is to show that the preconditioning of~\citet{karras2022elucidating} is
  a proper bound on the original log-likelihood \(\log p^^t_\psi(\mathbf{z}_t^0 | \mathbf{c}_t)\),
  given proper weighting of the objective.
In turn, we detail the optimization objective for the diffusion transitions for our model,
  specifically \cref{eq:term_diff_chain_body} from the main text.

To align with the framework from~\citet{karras2022elucidating}, we re-scale our forward process,
writing
\begin{align}
    \label{eq:diff_precond}
    \hat{\mathbf{z}}_t^k & = \frac{\tilde{\mathbf{z}}_t^k }{ \sqrt{\overline{\alpha}_k}} \\
                         & =
    \frac{\sqrt{\overline{\alpha}_k}\mathbf{z}_t^0 + \sqrt{1 - \overline{\alpha}_k}                      \epsilon}{ \sqrt{\overline{\alpha}_k}} \\
                         & = \mathbf{z}_t^0 + \sqrt{\frac{1 - \overline{\alpha}_k}{\overline{\alpha}_k}} \epsilon,
\end{align}
where we have the additive noise as
\begin{equation*}
    \sigma_k = \sqrt{\frac{1 - \overline{\alpha}_k}{\overline{\alpha}_k}}.
\end{equation*}

From~\citet{karras2022elucidating}, the preconditioned denoiser is written
as
\begin{equation*}
    D_\psi(\hat{\mathbf{z}}_t^k, \sigma_k, \mathbf{c}_t) = \cskip{\sigma_k}\hat{\mathbf{z}}_t^k +
      \cout{\sigma_k}F_{\psi}\left(\cin{\sigma_k}\hat{\mathbf{z}}_t^k, \cnoise{\sigma_k},
      \mathbf{c}_t\right).
\end{equation*}
where \( F_{\psi} \) is a neural network with parameters \( \psi \),
and the coefficients \( c \) are those used by~\citet{karras2022elucidating} (see Table 1, in the "Ours" column) with
\( \sigma_{\text{data}}=1 \).
For reference, these coefficients are
\begin{align*}
    \cskip{\sigma_k} & = \frac{1}{\sigma_k^2 + 1},        & \cout{\sigma_k} = \frac{\sigma_k}{\sqrt{\sigma_k^2 + 1}},
    \cin{\sigma_k}   & = \frac{1}{\sqrt{\sigma_k^2 + 1}}, & \cnoise{\sigma_k} = \frac{1}{4}\log(\sigma_k).
\end{align*}
To recover our original objective, we use that \( \epsilon = \frac{\hat{\mathbf{z}}_t^k - \mathbf{z}_t^0}{\sigma_k} \)
to recover the predicted noise,
\begin{align*}
    \epsilon_\psi\left(\tilde{\mathbf{z}}_t^k, k, \mathbf{c}_t\right) & =
    \frac{\hat{\mathbf{z}}_t^k - D_\psi(\hat{\mathbf{z}}_t^k, \sigma_k, \mathbf{c}_t)}{\sigma_k}.
\end{align*}

We now may find a new form for the noise prediction error term in \cref{eq:diff_elbo}.
We have
\begin{equation}
    \label{eq:noise_pred_error}
    \begin{aligned}
    w_k \left\| \epsilon - \epsilon_\psi\left(\tilde{\mathbf{z}}_t^k, k, \mathbf{c}_t\right) \right\|^2_2
     & = w_k \left\| \frac{\hat{\mathbf{z}}_t^k - \mathbf{z}_t^0}{\sigma_k} - \frac{\hat{\mathbf{z}}_t^k - D_\psi(\hat{\mathbf{z}}_t^k, \sigma_k, \mathbf{c}_t)}{\sigma_k} \right\|^2_2 \\
     & = \frac{w_k}{\sigma_k^2} \left\| D_\psi(\hat{\mathbf{z}}_t^k, \sigma_k, \mathbf{c}_t) - \mathbf{z}_t^0\right\|^2_2 \\
     & = \frac{w_k}{\sigma_k^2} \left\| \cout{\sigma_k}F_\psi + \cskip{\sigma_k}\hat{\mathbf{z}}_t^k  - \mathbf{z}_t^0 \right\|^2_2 \\
     & = \frac{w_k \cout{\sigma_k}^2}{\sigma_k^2} \left\| F_\psi -  \frac{\mathbf{z}_t^0 -  \cskip{\sigma_k}\hat{\mathbf{z}}_t^k}{\cout{\sigma_k}}  \right\|^2_2 ,
    \end{aligned}
\end{equation}
where to simplify notation going forward we will write the regression targets as
\begin{equation}
    \label{eq:diff_target_precond}
    \mathbf{y}(\mathbf{z}_t^0, k) = \frac{\mathbf{z}_t^0 -  \cskip{\sigma_k}\hat{\mathbf{z}}_t^k}{\cout{\sigma_k}},
\end{equation}
and the coefficient in front of the squared error as
\begin{equation*}
    \tilde{w}_k = \frac{w_k \cout{\sigma_k}^2}{\sigma_k^2}.
\end{equation*}

We substitute our value for \( \cout{\sigma_k}\) and obtain our diffusion loss for a given \( t \) as
\begin{equation}
    \label{eq:diff_elbo_precond}
    \mathcal{L}_{\text{diff}} = \sum_{k=1}^K \mathbb{E}_{\epsilon \sim \mathcal{N}(\mathbf{0},\mathbf{I})}\left[\tilde{w}_k \left\| F_{\psi}\left(\cdot\right) -  \mathbf{y}(\mathbf{z}_t^0, k)  \right\|^2_2 \right].
\end{equation}

Directly implementing the summation in \cref{eq:diff_elbo_precond} is prohibitively expensive since
  it requires a forward pass through the diffusion model \( K \) times per example.
Instead, we employ a Monte-Carlo approximation by sampling \( k \) from an arbitrary density \(
\rho(k) \) on \(\{1,\ldots,K\}\) such that
\begin{align}
    \label{eq:diff_mc_derive}
    \mathcal{L}_{\text{diff}} & = \sum_{k=1}^K \mathbb{E}_{                                                                              \epsilon \sim \mathcal{N}(\mathbf{0},\mathbf{I})}\left[\tilde{w}_k \left\| F_{\psi}\left(\cdot\right) -  \mathbf{y}(\mathbf{z}_t^0, k)  \right\|^2_2 \right] \\
                              & = \mathbb{E}_{k \sim \rho(k),                                                                            \epsilon \sim \mathcal{N}(\mathbf{0},\mathbf{I})}
    \left[\frac{\tilde{w}_k}{\rho(k)} \left\| F_{\psi}\left(\cdot\right) -  \mathbf{y}(\mathbf{z}_t^0, k)  \right\|^2_2 \right].
\end{align}
The choice of \( \rho(k) \) affects the variance of the Monte-Carlo estimate.
The coefficient \( \tilde{w}_k \) is typically viewed as a weight, where in our
case we have
\begin{align}
    \label{eq:is-weight}
    \tilde{w}_k & = \frac{w_k \cout{\sigma_k}^2}{\sigma_k^2} \\
                & = w_k \cdot \overline{\alpha}_k \\
                & = \frac{\beta_k\overline{\alpha}_k}{2\alpha_k(1-\overline{\alpha}_k)}.
\end{align}


Many works drop the weights and perform uniform sampling, implicitly reweighting the objective.
Dropping this term imbalances the terms in the ELBO.
\Citet{vahdat2021score} also trains a diffusion model end-to-end in a VAE framework, and
highlight this issue as well; during training of the encoder,~\citet{vahdat2021score} choose to retain the weights
so that the encoder is encouraged to match the true posterior of the diffusion model.
We provide an alternative perspective on this issue by noting that the weights \( \tilde{w}_k \)
  are very large for small \( k \) (small noise levels), and very small for large \( k \) (large
  noise levels).
Intuitively, this means that maximizing the ELBO penalizes errors in the noise prediction at small
  noise levels, where the signal-to-noise ratio is high.
On the other hand, errors in noise prediction at large noise levels are not penalized as much,
  where the diffusion model is focused on denoising the lower-frequency structure of the sample.
In the context of our full model, this property implies that the strongest regularization of the
  encoder comes from the small noise levels, penalizing the encoder for producing latent samples that
  cannot reveal the fine-grained structure of the data.
We retain the weights in our implementation, as we do not separately train the diffusion model from
  the encoder.
Lastly, we set \(\rho(k) = \frac{1}{K}\) for uniform sampling, to ensure that the Monte-Carlo
  estimate is unbiased without requiring any additional weighting of the loss.

\section{Experiment Details}\label{app:experiments}
To ensure a fair comparison for our synthetic experiments, we fix architectural hyperparameters
  across the DKF and DDSSM models, and tune hyperparameters for both models using
  Optuna~\citep{akiba2019optuna}.
The only difference in architecture is the implementation for the transition model, where the DKF
  uses a Gaussian transition module and the DDSSM uses a conditional U-Net parameterized diffusion
  model.
We provide identical search spaces to the tuning algorithm, using the one-step JSD as the
  evaluation metric on a validation set of examples.

\subsection{Synthetic Experiment Architectural Parameters} \label{app:architectural_params}

We construct our models using the modular elements described in \cref{apd:model}.
Specifically, we define the \textit{time mixer} (operating over the sequence length axis) and the
  \textit{feature mixer} (operating across the latent or observation dimensions) for each individual
  module.
Across all synthetic experiments, we fix the latent dimension to $M=1$, the sequential latent lag
  to $j=1$, the context channels to 16, and the dense hidden dimension to 16.
Embeddings for continuous time steps are disabled, while discrete mask embeddings have dimension 8.

For both the baseline DKF and our DDSSM, the shared contextual architectures are summarized in
  \cref{tab:base_architectures}.

\begin{table}[ht]
    \centering
    \begin{tabular}{l c c c l l}
        \toprule
        \textbf{Module}           & \textbf{Layers} & \textbf{Channels} & \textbf{Hidden Dim} & \textbf{Time Mixer} & \textbf{Feature Mixer} \\
        \midrule
        Encoder Summary           & 3               & --                & 16                  & GRU                 & --                     \\
        Encoder Context           & 2               & 16                & 16                  & Conv ($k=2$)        & Conv ($k=2$)           \\
        Decoder Context           & 2               & 16                & 16                  & Conv ($k=2$)        & Conv ($k=2$)           \\
        Initialization Context    & 2               & 16                & 16                  & Identity            & Conv ($k=2$)           \\
        Gaussian Transition (DKF) & 3               & 16                & 16                  & Identity            & Conv ($k=2$)           \\
        \bottomrule
    \end{tabular}
    \caption{Shared architectural parameters across models, including the baseline DKF transition model.
        The Gaussian transition relies on an identity time mixer since it only projects summary statistics
          for a single target step, whereas the encoder and decoder leverage temporal convolutions to process
          multi-step histories and future summaries.
    } \label{tab:base_architectures}
\end{table}

\textbf{Diffusion Transition Specifications.}
For the DDSSM, we substitute the Gaussian transition module with a conditional U-Net parameterized
  diffusion model.
Unlike the Gaussian baseline, the diffusion denoiser utilizes a convolutional time mixer rather
  than an identity map.
This is because at each diffusion step $k$, the transition model observes both the clean history
  states $\mathbf{z}_{t-j:t-1}$ and the noisy target state $\mathbf{z}_t^k$, requiring explicit
  temporal mixing.
We summarize the network sizing and diffusion schedule parameters in Table
  \ref{tab:diff_architectures}.

\begin{table}[ht]
    \centering
    \begin{tabular}{l l}
        \toprule
        \textbf{Parameter}              & \textbf{Value}        \\
        \midrule
        U-Net Layers                    & 3                     \\
        U-Net Channels (Blocks)         & 16                    \\
        Time Mixer                      & Conv ($k=2$)          \\
        Feature Mixer                   & Conv ($k=2$)          \\
        Diffusion Step Embedding Space  & 64 (projected to 128) \\
        Feature Embedding Dimension     & 8                     \\
        \midrule
        \textbf{Diffusion Schedule}     &                       \\
        Number of Diffusion Steps ($K$) & 128                   \\
        Noise Schedule                  & Linear                \\
        \bottomrule
    \end{tabular}
    \caption{Architectural and schedule parameters specific to the diffusion-based transition model in DDSSM.} \label{tab:diff_architectures}
\end{table}

\subsubsection{Hyperparameters and Tuning} \label{app:tuning}
We use Optuna~\citep{akiba2019optuna} for hyperparameter tuning.
Each trial consists of training a model for 1500 steps and evaluating the CRPS-sum metric on unseen
  validation data.
We use 1024 training examples, 1024 validation examples, and 1024 test examples for the synthetic
  experiment, each generated by simulating the same underlying process with a fixed random seed.

The parameters we tune, their tuning ranges, and the optimal parameters found for the DKF and DDSSM
  models are summarized in \cref{tab:tuning_summary,tab:fixed_hyperparameters}.

\begin{table}[ht]
    \centering
    \begin{tabular}{|l|l|l|l|}
        \hline
        \textbf{Parameter}       & \textbf{Tuning Range}               & \textbf{DKF} & \textbf{DDSSM} \\
        \hline
        lambda\_schedule         & linear, cosine                      & cosine       & cosine         \\
        lambda\_warmup\_steps    & 200 -- 1200                         & 867          & 889            \\
        lambda\_end              & 0.3 -- 2                            & 1.245        & 1.243          \\
        enc\_lr                  & $5\times10^{-5}$ -- $10^{-2}$ (log) & 0.00887      & 0.000864       \\
        dec\_lr                  & $5\times10^{-5}$ -- $10^{-2}$ (log) & 0.00777      & 0.000300       \\
        trans\_lr                & $5\times10^{-5}$ -- $10^{-2}$ (log) & 0.000099     & 0.00941        \\
        zinit\_lr                & $10^{-4}$ -- $10^{-3}$ (log)        & 0.005        & 0.000801       \\
        S                        & 1 -- 4                              & 2            & 2              \\
        batch\_size              & 64, 128, 256, 512                   & 512          & 128            \\
        transition.schedule.S\_k & 1 -- 8 (Diffusion only)             & 1            & 5              \\
        \hline
    \end{tabular}
    \caption{Tuning ranges and optimal parameters found for DKF and DDSSM models.}
    \label{tab:tuning_summary}
\end{table}

\begin{table}[ht]
    \centering
    \begin{tabular}{|l|l|}
        \hline
        \textbf{Hyperparameter}  & \textbf{Fixed Value}        \\
        \hline
        Number of training steps & 1500 (tuning), 3000 (final) \\
        Number of Optuna trials  & 50                          \\
        Number of startup trials & 10                          \\
        Encoder/decoder channels & 16                          \\
        Encoder/decoder layers   & 2                           \\
        Weight decay             & 0.01                        \\
        \hline
    \end{tabular}
    \caption{Hyperparameters fixed across all models during tuning.}
    \label{tab:fixed_hyperparameters}
\end{table}

\section{Variational Hierarchical Prior} \label{app:vhp}
\TODO{add covariates to this section, if deemed appropriate}
\Citet{klushyn2021latent, klushyn2019learning}, consider the problem of defining the prior distribution over the initial
latent state \(p(\mathbf{z}_1)\) in the \(j=1\) case.
Typically, the prior is chosen by DSSMs to be a standard Gaussian~\citet{krishnan2015deep,
      krishnan2017structured, karl2016deep}.
The prior over \(\mathbf{z}_1\) however regularizes the rest of the latent trajectory, as both the
  transition model and the encoder posterior are conditioned on \(\mathbf{z}_1\).
This is experimentally validated by~\citet{klushyn2019learning, klushyn2021latent}, where it is
  demonstrated that a more flexible prior over the initial latent state can lead to better inference
  about the latent trajectory.

As our diffusion transition model can be highly expressive, we wish to avoid inhibiting the
  expressiveness of the transition model by imposing a simple prior over the initial latent states.
We therefore follow~\citet{klushyn2021latent} and use an expressive initial distribution via a
  variational hierarchical prior (VHP).

In our case, we wish to handle higher Markov orders, so our initial state distribution is over the
  first \(j\) latents, \( p(\mathbf{z}_{1 : j}) \).
For the \( j=1 \) case,~\cite{klushyn2021latent} uses a two-level hierarchical prior, where the initial latent state
\( \mathbf{z}_1 \) is modeled as
\begin{equation*}
    p(\mathbf{z}_1) = \int p_\eta(\mathbf{z}_1 | \mathbf{z}_0) p(\mathbf{z}_0) d\mathbf{z}_0,
\end{equation*}
where \( p(\mathbf{z}_0) \) is a standard Gaussian.
This choice mimics the fact that the optimal empirical bayes prior for a VAE is \( p*(\mathbf{z}_1)
  = \mathbb{E}_{p(\mathbf{x}_{1 : T})}[q_\phi(\mathbf{z}_1 | \mathbf{x}_{1: T})]\).

We generalize this to the \( j \)-order Markovian setting by introducing \( j \) auxiliary
  variables \( \mathbf{z}_{-j+1 : 0} \).
The conditional distribution \( p_\eta(\mathbf{z}_{1 : j} | \mathbf{z}_{-j+1 : 0}) \) is
given by the chain rule,
\begin{align*}
    p_\eta(\mathbf{z}_{1 : j} | \mathbf{z}_{-j+1 : 0}) & =
    \prod_{t=1}^j p_\eta(\mathbf{z}_t | \mathbf{z}_{t-1}, \ldots, \mathbf{z}_{-j + 1}) \\
                                                       & = \prod_{t=1}^j p_\eta(\mathbf{z}_t | \mathbf{z}_{t-j : t-1}),
\end{align*}
where we have used the \( j \)-order Markovian property in the second line.
Marginalizing over the auxiliary variables, we have
\begin{align*}
    p_\eta(\mathbf{z}_{1 : j}) & = \int p_\eta(\mathbf{z}_{1 : j} | \mathbf{z}_{-j+1 : 0}) p(\mathbf{z}_{-j+1 : 0}) d\mathbf{z}_{-j+1 : 0} \\
                               & = \int \prod_{t=1}^j p_\eta(\mathbf{z}_t | \mathbf{z}_{t-j : t-1}) p(\mathbf{z}_{-j+1 : 0}) d\mathbf{z}_{-j+1 : 0}.
\end{align*}
Continuing, we introduce the variational distribution \( q_\Phi(\mathbf{z}_{-j+1 : 0} |
  \mathbf{z}_{1 : j}) \), parameterized by \( \Phi \), to approximate the posterior over the
  auxiliary variables.
We find the bound on the optimal expected marginal log-likelihood of \( p(\mathbf{z}_{1 : j}) \) as
\begin{align*}
     & \mathbb{E}_{p*(\mathbf{z}_{1 : j})}[\log p_\eta(\mathbf{z}_{1 : j})] \\
     & \geq \mathbb{E}_{p(\mathbf{x}_{1 : T})}\mathbb{E}_{q_\phi(\mathbf{z}_{1 : T} | \mathbf{x}_{1 : T}, \mathbf{u}_{1 : T})}
    \left[
        \mathbb{E}_{q_\Phi(\mathbf{z}_{-j+1 : 0} | \mathbf{z}_{1 : j})}
        \left[
            \log
            \frac{
                p_\eta(\mathbf{z}_{1 : j} | \mathbf{z}_{-j+1 : 0}) p(\mathbf{z}_{-j+1 : 0})
            }{
                q_\Phi(\mathbf{z}_{-j+1 : 0} | \mathbf{z}_{1 : j})
            }
            \right]
    \right] \\
     & =
    \mathbb{E}_{p(\mathbf{x}_{1 : T})}\mathbb{E}_{q_\phi(\mathbf{z}_{1 : T} | \mathbf{x}_{1 : T}, \mathbf{u}_{1 : T})}
    \left[
        \mathbb{E}_{q_\Phi(\mathbf{z}_{-j+1 : 0} | \mathbf{z}_{1 : j})}
        \left[
            \log
            p_\eta(\mathbf{z}_{1 : j} | \mathbf{z}_{-j+1 : 0})
            + \log p(\mathbf{z}_{-j+1 : 0})
            - \log q_\Phi(\mathbf{z}_{-j+1 : 0} | \mathbf{z}_{1 : j})
            \right]
    \right] \\
     & =
    \mathbb{E}_{p(\mathbf{x}_{1 : T})}\mathbb{E}_{q_\phi(\mathbf{z}_{1 : T} | \mathbf{x}_{1 : T}, \mathbf{u}_{1 : T})}
    \Bigg[
        \mathbb{E}_{q_\Phi(\mathbf{z}_{-j+1 : 0} | \mathbf{z}_{1 : j})}
        \left[
            \sum_{t=1}^j \log p_\eta(\mathbf{z}_t | \mathbf{z}_{t-j : t-1})
    \right] \nonumber \\
     & \qquad\qquad\qquad\qquad\qquad\qquad- \KL{q_\Phi(\mathbf{z}_{-j+1 : 0} | \mathbf{z}_{1 : j})}{p(\mathbf{z}_{-j+1 : 0})}
        \Bigg],
\end{align*}
where in the last line we have expanded the conditional prior and recognized the KL divergence.
In line with~\cite{klushyn2021latent}, we set \( p(\mathbf{z}_{-j+1 : 0}) = \mathcal{N}(0,
  \mathbf{I}) \), and parameterize \( q_\Phi(\mathbf{z}_{-j+1 : 0} | \mathbf{z}_{1 : j}) \) as a
  diagonal Gaussian with mean and variance given by neural networks that take \( \mathbf{z}_{1 : j}
  \) as input.

Note that in the variational posterior, \( q_\phi(\mathbf{z}_{1 : j} | \mathbf{x}_{1 : T},
  \mathbf{u}_{1 : T}) \) depends on all observations and inputs, while in the prior \(
  p_\eta(\mathbf{z}_{1 : j} | \mathbf{z}_{-j+1 : 0}) \) only depends on the auxiliary variables.
Therefore it is unnecessary to provide a hierarchical structure to the initial latent states for
  the variational posterior, as the variational posterior can already leverage the full context of
  the observations and inputs to infer the initial latent states.
As we implement a network to parameterize \( q^^t_\phi(\mathbf{z}_{t} | \mathbf{z}_{t-j : t-1},
  \mathbf{x}_{t:T}, \mathbf{u}_{t-j : t}) \), we can simply use the same network to parameterize \(
  q_\phi(\mathbf{z}_{1 : j} | \mathbf{x}_{1 : T}, \mathbf{u}_{1 : T}) \) by providing zero inputs
  for \( \mathbf{z}_{t-j : t-1} \) when \( t \leq j \).
These drawn latent states are then used for drawing from \( q_\Phi(\mathbf{z}_{-j+1 : 0} |
  \mathbf{z}_{1 : j}) \) to compute the hierarchical prior.
\TODO{describe more precisely how this is implemented}

Rewriting \( \mathcal{L}_{\text{init}}(\phi,\eta; \mathbf{x}_{1 : T}, \mathbf{u}_{1 : T}) \) from
\cref{eq:term_init_loss} using the hierarchical prior, we have
\begin{align}
    \mathcal{L}_{\text{init}}(\phi,\eta,\Phi; \mathbf{x}_{1 : T}, \mathbf{u}_{1 : T})
     & = \mathbb{E}_{q_\phi(\mathbf{z}_{1 : T} | \mathbf{x}_{1: T}, \mathbf{u}_{1 : T})}
    \Big[
    \log q_\phi\left(\mathbf{z}_{1 : j} | \mathbf{x}_{1 : T}, \mathbf{u}_{1 : t}\right) \nonumber \\
     & \qquad- p_\eta(\mathbf{z}_{1 : j} | \mathbf{z}_{-j+1 : 0}) p(\mathbf{z}_{-j+1 : 0}) d\mathbf{z}_{-j+1 : 0}
    \Big] \nonumber \\
     & \leq \mathbb{E}_{q_\phi(\mathbf{z}_{1 : T} | \mathbf{x}_{1: T}, \mathbf{u}_{1 : T})}
    \Big[
    \log q_\phi\left(\mathbf{z}_{1 : j} | \mathbf{x}_{1 : T}, \mathbf{u}_{1 : t}\right) \nonumber \\
     & \qquad- \mathbb{E}_{q_\Phi(\mathbf{z}_{-j+1 : 0} | \mathbf{z}_{1 : j})}
        \left[
            \log p_\eta(\mathbf{z}_{1 : j} | \mathbf{z}_{-j+1 : 0})
            + \log p(\mathbf{z}_{-j+1 : 0})
            - \log q_\Phi(\mathbf{z}_{-j+1 : 0} | \mathbf{z}_{1 : j})
            \right]
    \Big] \nonumber \\
     & =\underbrace{ \mathbb{E}_{q_\phi(\mathbf{z}_{1 : T} | \mathbf{x}_{1: T}, \mathbf{u}_{1 : T})}
        \Big[
    \log q_\phi\left(\mathbf{z}_{1 : j} | \mathbf{x}_{1 : T}, \mathbf{u}_{1 : t}\right)\Big]}_{\text{negative entropy}} \nonumber \\
     & \qquad +  \mathbb{E}_{q_\phi(\mathbf{z}_{1 : T} | \mathbf{x}_{1: T}, \mathbf{u}_{1 : T})}
    \Big[\mathbb{E}_{q_\Phi(\mathbf{z}_{-j+1 : 0} | \mathbf{z}_{1 : j})}
        \Big[
    -\log p_\eta(\mathbf{z}_{1 : j} | \mathbf{z}_{-j+1 : 0}) \nonumber \\
     & \qquad \qquad
            - \log p(\mathbf{z}_{-j+1 : 0})
            + \log q_\Phi(\mathbf{z}_{-j+1 : 0} | \mathbf{z}_{1 : j})
            \Big]
        \Big]
    \Big] \nonumber \\
     & =\underbrace{ \mathbb{E}_{q_\phi(\mathbf{z}_{1 : T} | \mathbf{x}_{1: T}, \mathbf{u}_{1 : T})}
        \Big[
    \log q_\phi\left(\mathbf{z}_{1 : j} | \mathbf{x}_{1 : T}, \mathbf{u}_{1 : t}\right)\Big]}_{\mathcal{L}_{\text{init}}^{\text{entropy}}(\phi; \mathbf{x}_{1 : T}, \mathbf{u}_{1 : T})} \nonumber \\
     & \qquad +  \mathbb{E}_{q_\phi(\mathbf{z}_{1 : T} | \mathbf{x}_{1: T}, \mathbf{u}_{1 : T})}
    \Big[\mathbb{E}_{q_\Phi(\mathbf{z}_{-j+1 : 0} | \mathbf{z}_{1 : j})}
        \Big[
    - \sum_{t=1}^j \log p_\eta(\mathbf{z}_t | \mathbf{z}_{t-j : t-1}) \nonumber \\
     & \qquad \qquad
            + \KL{q_\Phi(\mathbf{z}_{-j+1 : 0} | \mathbf{z}_{1 : j})}{p(\mathbf{z}_{-j+1 : 0})}
            \Big]
        \Big]
    \Big] \nonumber \\
     & := \mathcal{L}_{\text{init}}^{\text{entropy}}(\phi; \mathbf{x}_{1 : T}, \mathbf{u}_{1 : T})
    + \mathcal{L}_{\text{VHP}}(\phi,\eta,\Phi; \mathbf{x}_{1 : T}, \mathbf{u}_{1 : T}) \label{eq:vhp_loss},
\end{align}
and notice that we have split the final line into two terms, the negative entropy of the variational
posterior over the initial states, and the variational hierarchical prior (VHP) loss,
where we have defined
\begin{align*}
    \mathcal{L}_{\text{VHP}}(\phi,\eta,\Phi; \mathbf{x}_{1 : T}, \mathbf{u}_{1 : T})
     & =  \mathbb{E}_{q_\phi(\mathbf{z}_{1 : T} | \mathbf{x}_{1: T}, \mathbf{u}_{1 : T})}
    \Big[\mathbb{E}_{q_\Phi(\mathbf{z}_{-j+1 : 0} | \mathbf{z}_{1 : j})}
        \Big[
    - \sum_{t=1}^j \log p_\eta(\mathbf{z}_t | \mathbf{z}_{t-j : t-1}) \\
     & \qquad \qquad
            + \KL{q_\Phi(\mathbf{z}_{-j+1 : 0} | \mathbf{z}_{1 : j})}{p(\mathbf{z}_{-j+1 : 0})}
            \Big]
        \Big] .
\end{align*}

Performing a substitution of the resulting inequality in \cref{eq:vhp_loss} into the ELBO in
\cref{eq:term_init_loss}, we obtain the ELBO which incorporates the variational hierarchical prior:
\begin{align}
     & -\log p\left(\mathbf{x}_{1:T} | \mathbf{u}_{1 : T}\right) \nonumber \\
     & \leq  \mathcal{L}_{\text{init}}^{\text{entropy}}(\phi; \mathbf{x}_{1 : T}, \mathbf{u}_{1 : T})
    + \mathcal{L}_{\text{VHP}}(\phi,\eta,\Phi; \mathbf{x}_{1 : T}, \mathbf{u}_{1 : T}) \nonumber \\
     & +  \mathcal{L}_{\text{recon}}(\phi,\theta; \mathbf{x}_{1 : T}, \mathbf{u}_{1 : T})
    +  \mathcal{L}_{\widehat{\text{diff}}}(\phi,\psi; \mathbf{x}_{1 : T}, \mathbf{u}_{1 : T})
    +  \mathcal{L}_{\text{entropy}}(\phi; \mathbf{x}_{1 : T}, \mathbf{u}_{1 : T}). \label{eq:final_elbo_vhp}
\end{align}
which we minimize with respect to the parameters \( \phi, \eta, \theta, \psi, \Phi \).

\section{Model Architecture Details}\label{apd:model}

The method of~\citet{tashiro2021csdi}, titled ``CSDI: Conditional Score-based Diffusion Models for
  Probabilistic Time Series Imputation", uses a well-established architecture for modeling
  time-series data with diffusion models~\citet{tashiro2021csdi}\AFTER{cite diffwave}, and we adapt
  it to our setting of autoregressive generation in the latent space.
We choose to extend the architecture of~\citet*{tashiro2021csdi}, for several reasons.
First, it is a well-established architecture for modeling time-series data with diffusion models,
  Second, its success in modeling directly in observation space suggests that it is sufficiently
  powerful to model the transition distribution in latent space, and it is simple enough to allow us
  to modularize the architecture depending on the specific needs of the encoder, decoder, and
  transition modules.
Specifically, we extend the U-net architecture from CSDI to not only our diffusion model, but also
  to the encoder, decoder, and initialization modules, which all operate on latent histories.

\subsection{U-net Architecture} \label{app:unet_architecture}
Because we operate in the latent space and generate autoregressively, we utilize a static mask
  indicating which time steps functionally serve as the conditioning history versus the generated
  output, rather than a dynamic missing-data mask.
Our mask is thus a binary mask \(\mathbf{m}_{\mathrm{hist}}\) of shape \( (j+1) \), which equals
  one for history steps and zero for the target step, since we are always generating a single step
  and conditioning on \(j\) history steps.
The input to the denoiser is a tensor \(\mathbf{Z} \in \mathbb{R}^{M \times (j+1)}\), where \(M\)
  is the latent dimension.
The target noise is always located at the last time step, while the clean history is located in the
  preceding \(j\) time steps.

\TODO{gru cite}
\paragraph{Input Projection}
Following~\citet{tashiro2021csdi}, we separate the input into two components via the mask: a clean
  history \(\mathbf{Z}_{\mathrm{hist}} = \mathbf{Z} \odot \mathbf{m}_{\mathrm{hist}}\) and a noisy
  target \(\mathbf{Z}_{\mathrm{noisy}} = \mathbf{Z} \odot (1 - \mathbf{m}_{\mathrm{hist}})\).
These are concatenated along a new channel dimension to form a tensor in \(\mathbb{R}^{2 \times M
      \times (j+1)}\).
A \(1\times 1\) convolution and a ReLU activation project this input to a hidden channel dimension
  \(C\), yielding an initial representation \(\mathbf{H}_0 \in \mathbb{R}^{C \times M \times
      (j+1)}\).

\paragraph{Extra Conditioning Information}
Conditioning is injected via a continuous diffusion timestep embedding and a side information
  tensor.
The scalar diffusion noise level is embedded via sinusoidal positional encodings and a two-layer
  multi-layer perceptron into \(\mathbf{e}_{\mathrm{diff}} \in \mathbb{R}^{E_{\mathrm{diff}}}\).
The side information tensor \(\mathbf{S} \in \mathbb{R}^{E_{\mathrm{side}} \times M \times (j+1)}\)
  concatenates dynamic covariates \(\mathbf{u}_{t-j: t}\), absolute time embeddings, and the mask
  \(\mathbf{m}_{\mathrm{hist}}\).
Since this diffusion operates on compressed latents rather than the original explicitly meaningful
  observation coordinates, in contrast to~\citet{tashiro2021csdi}, we do not utilize static feature
  embeddings here.

\paragraph{Residual Blocks}
The core of our network consists of \(L\) stacked residual blocks, each of which takes in a
  representation of shape \((C, M, j+1)\) and produces an output of the same shape as well as a skip
  connection.
The residual block itself is mostly unchanged from CSDI, except we abstract the time-mixing and
  feature-mixing operations.
This allows us to replace computationally heavy attention operations with architectures like a 1D
  convolution stack or a Gated Recurrent Unit (GRU) when the sequence length \(j\) is small.
We refer the reader to~\citet{tashiro2021csdi} for more details on the architecture of the residual
  block.

\paragraph{Output Projection}
We continue to follow CSDI for handling the residual block inputs and outputs, as well as the skip
  connections.
However, we deviate from CSDI in the output projection step.
Because we generate autoregressively, we only need to produce an output for the target time step,
  rather than a full sequence of length \(j+1\).
So, we slice the final output of the residual block stack to only keep the target time step,
  yielding a tensor of shape \((M)\), containing the predicted noise for the target time step.

\subsection{Context Producer} \label{app:context_producer_architecture}
The encoder, decoder and initialization module, similar to the denoiser, operate on latent
  histories.
Therefore for the encoder, decoder, and initialization module, we adapt the U-net architecture
  above as a general input-output architecture, where the inputs and outputs are tensors of shape
  \((h,L)\) for some hidden dimension \(h\) and sequence length \(L\).

The encoder, decoder, and initialization modules process sequential latent and observable data
  histories.
To support this repetition, we abstract the U-net framework into a generalized feature-extraction
  module, which we call the \textit{Context Producer}.
The Context Producer accepts inputs of shape \((H_{\mathrm{seq}}, L)\), where \(H_{\mathrm{seq}}\)
  is the spatial or feature dimension and \(L\) is the temporal sequence length.
The output is a vector representing the entire trajectory.

\paragraph{Input Projection}
Unlike the denoiser which separates clean and noisy data, the Context Producer takes a single
  sequence tensor \(\mathbf{H}_{\mathrm{in}} \in \mathbb{R}^{H_{\mathrm{seq}} \times L}\).
We treat this input sequence as having a single channel, flattening it temporarily to shape \((1,
  H_{\mathrm{seq}} \times L)\), and apply a \(1 \times 1\) convolution followed by a ReLU activation.
This projects the input to a hidden channel dimension \(C\), yielding an initial representation
  \(\mathbf{H}_0 \in \mathbb{R}^{C \times H_{\mathrm{seq}} \times L}\).

\paragraph{Extra Conditioning Information}
Because the Context Producer does not reverse a diffusion process, we completely omit the diffusion
  step embedding.
As with the denoiser, we inject conditioning information via a side information tensor, but we
  adapt the construction of this tensor to a side-information tensor \(\mathbf{S} \in
  \mathbb{R}^{E_{\mathrm{side}} \times H_{\mathrm{seq}} \times L}\), where \(E_{\mathrm{side}}\) is
  the side information embedding dimension.
We carefully construct this tensor to support both temporal and static covariates, as well as
  absolute time embeddings and the history/target mask.
To construct the tensor we handle temporal and static information separately, and then concatenate
  them along the feature dimension.
For each type of information we have:
\begin{itemize}
    \item \textbf{Temporal Information.}
          Dynamic covariates, time embeddings, and valid/padding masks vary over time \(L\) but are uniform
            across the feature dimension \(H_{\mathrm{seq}}\).
          These are computed as a tensor of shape \((E_{\mathrm{time}} + E_{\mathrm{mask}}, L)\) and expanded
            along the feature dimension to \((E_{\mathrm{time}} + E_{\mathrm{mask}}, H_{\mathrm{seq}}, L)\).
    \item \textbf{Static Information.}
          Static categorical or contextual embeddings vary across the features but are constant across time.
          These are mapped to a tensor \((E_{\mathrm{static}}, H_{\mathrm{seq}})\) and expanded along the
            time dimension to \((E_{\mathrm{static}}, H_{\mathrm{seq}}, L)\).
\end{itemize}
We concatenate the temporal and static information along the feature dimension to yield the final
  side information tensor \(\mathbf{S} \in \mathbb{R}^{E_{\mathrm{side}} \times H_{\mathrm{seq}}
  \times L}\) where \(E_{\mathrm{side}} = E_{\mathrm{time}} + E_{\mathrm{mask}} +
  E_{\mathrm{static}}\).

\paragraph{Residual Blocks}
The residual blocks of the context producer are unchanged from \cref{app:unet_architecture}, except
  the diffusion embedding is omitted.
The utilization of the side information tensor is unchanged.

\paragraph{Output Projection}
The rest of the architecture is unchanged from our denoiser adaptation, except we do not project
  the \((C,h,L)\) output back to \((1,h,L)\) before slicing as we wish to produce a dense summary of
  the input rather than a noise prediction for a single time step.
To concisely collapse the temporal dimension, we reshape the tensor equivalently to \((C \times
  H_{\mathrm{seq}}, L)\) and apply a 1D convolution over the time dimension.
Standard 1D convolutions would mix representations across the feature dimension, so we strictly use
  a grouped 1D convolution with groups equal to \(C \times H_{\mathrm{seq}}\) and a kernel size equal
  to the sequence length \(L\).
This operation pools the sequence, outputting a flattened context vector in \(\mathbb{R}^{C \times
  H_{\mathrm{seq}}}\).

\subsection{Decoder Architecture} \label{app:decoder_architecture}
The decoder architecture parameterizes the emissions density \(p^^t_{\theta}(\mathbf{x}_t |
  \mathbf{z}_{t-j+1:t}, \mathbf{u}_{t-j + 1 : t })\) as a diagonal Gaussian distribution over the
  observation \(\mathbf{x}_t \in \mathbb{R}^D\).
It takes the sequential latent history and extracts a summary using a Context Producer to predict
  the observation parameters.

\paragraph{Input Preparation}
The inputs include the length-\(j\) latent history \(\mathbf{z}_{t : t-j + 1}\) and covariates
  \(\mathbf{u}_{t-j + 1: t}\).
For early time steps where \(t < j\), the available history is left-padded with zero vectors to
  strongly enforce a fixed sequence length \(j\).
We include a binary mask to indicate which time steps are padded versus valid, and add this to the
  set of conditioning information.

The resulting sequence of shape \((M, j)\) is processed step-wise by a linear layer mapping the
  latent dimension \(M\) to the Context Producer's spatial dimension \(H_{\mathrm{seq}}\).
This outputs the primary input sequence \(\mathbf{H}_{\mathrm{in}} \in \mathbb{R}^{H_{\mathrm{seq}}
  \times j}\) for the Context Producer.

\paragraph{Conditioning Construction}
The decoder's Context Producer requires side information aligned to its \(j\)-step temporal axis:
\begin{itemize}
    \item \textbf{Temporal Information:}
          We compute absolute time embeddings for the \(j\) history steps.
          These are concatenated with the step-aligned dynamic covariates \(\mathbf{u}_{t-j+1:t}\).
          As previously mentioned, we embed a binary padding mask---which flags valid history vectors versus
            left-padded zero vectors---into a continuous embedding space.
    \item \textbf{Static Information:}
          If static observation-level covariates are present, they are linearly projected from the data
            dimension \(D\) to the spatial hidden dimension \(H_{\mathrm{seq}}\).
          The purpose of this projection is to ease the decoder's task of matching its generated data to
            observed data, since the semantically meaningful static covariates are in the data space rather
            than the latent space.
\end{itemize}
These elements are integrated as the side information tensor \(\mathbf{S}\) utilizing the
  mechanisms defined in \cref{app:context_producer_architecture}.

\paragraph{Gaussian Output}
The Context Producer condenses the sequence into a final structured summary vector of shape \( (C
  \times H_{\mathrm{seq}}) \).
Because the Context Producer utilizes several non-linear residual blocks, we preserve a residual
  path for available static covariates by mapping the flattened static covariates through an
  auxiliary linear layer and adding the result directly to the summary vector.
This resulting context vector is passed into the Gaussian output head, which comprises two
  multi-layer perceptrons producing the distribution parameters: the mean \(\mu_\theta(\mathbf{r}_t)
  \in \mathbb{R}^D\) and the log-variance \(\log \Sigma_\theta(\mathbf{r}_t) \in \mathbb{R}^D\) for
  the components of \(\mathbf{x}_t\).

\subsection{Encoder Architecture} \label{app:encoder_architecture}
The encoder approximates the smoothed variational density \(q^^t_\phi(\mathbf{z}_t |
  \mathbf{z}_{t-j:t-1}, \mathbf{x}_{1:T}, \mathbf{u}_{t-j : t})\).
As mentioned in \cref{app:post_fact}, we simplify the conditioning set by summarizing all future
  information into a sequence of continuous vectors \(\mathbf{h}_{1:T}\), i.e. our variational family
  assumes \(q^^t_\phi(\mathbf{z}_t | \mathbf{z}_{t-j : t-1}, \mathbf{u}_{t-j : t}, \mathbf{x}_{t :
      T}) \approx q^^t_\phi(\mathbf{z}_t | \mathbf{z}_{t-j : t-1}, \mathbf{u}_{t-j : t}, \mathbf{h}_t)\).
Therefore, there are two stages to our encoder's operation: (i) a full-sequence \textit{Future
      Summary} module that processes the entire future trajectory to produce a sequence of future
  summaries \(\mathbf{h}_{1:T}\), and (ii) an autoregressive \textit{Context Producer} that takes in
  the current latent history and the current future summary to produce the parameters of the
  variational distribution over \(z_t\).
We provide the algorithmic details of sampling a full trajectory from the encoder in
  \cref{alg:sample_latent_paths}.

\paragraph{Future Summary Module}
The Future Summary module computes \(\mathbf{h}_{1:T} = F_{\phi}(\mathbf{x}_{1:T},
  \mathbf{m}_{\mathrm{obs}})\).
At each time step across the length-\(T\) sequence, we concatenate the observations
  \(\mathbf{x}_t\), absolute time embeddings, observation missingness masks
  \(\mathbf{m}_{\mathrm{obs}, t}\), and the flattened static covariates.
This feature vector is linearly projected to a hidden dimension \(C_{\mathrm{summary}}\).
In line with~\citet{krishnan2017structured}, to ensure that \(\mathbf{h}_t\) only absorbs
  information from the \textit{future and present} (time\(t \dots T\)), the entire projected sequence
  is reversed along the time dimension and passed through a causal sequence model (such as a Gated
  Recurrent Unit or a causally-masked Transformer).
The output sequence is reversed again, restoring the original chronological order while
  guaranteeing that \(\mathbf{h}_t \in \mathbb{R}^{C_{\mathrm{summary}}}\) depends only on sequence
  indices \(\ge t\).

\paragraph{Input Preparation}
Given the future summary \(\mathbf{h}_t\), the encoder autoregressively infers the parameters for
  \(\mathbf{z}_t\) contingent on previously sampled states \(\mathbf{z}_{t-j : t-1}\).
For the first \(j\) time steps, we left-pad the latent history with zero vectors and include a
  binary mask to indicate which time steps are padded versus valid, and add this to the set of
  conditioning information.
This is the same strategy taken by the decoder (see \cref{app:decoder_architecture}).

We linearly project the future summary from \(C_{\mathrm{summary}}\) to the Context Producer
  spatial dimension \(H_{\mathrm{seq}}\).
Likewise, the \(j\) history vectors are mapped from the latent dimension \(M\) to
  \(H_{\mathrm{seq}}\).
These are concatenated along the temporal axis to yield the primary input tensor
  \(\mathbf{H}_{\mathrm{in}} \in \mathbb{R}^{H_{\mathrm{seq}} \times (j+1)}\) for the Context
  Producer, where the future summary is located at index \(0\) and the history vectors are located at
  indices \(1 \dots j\).
By placing the future summary at the beginning of the sequence, we allow the Context producer to
  attend to the future summary at each layer of its architecture.
This choice is once again inspired by~\citet{krishnan2017structured}, who project the future
  summary to the initial hidden state of their RNN-based encoder, which allows the future summary to
  influence the entire inference process.
In fact, this is the behavior of the Context Producer when we use an RNN for the time-mixing
  operation in the residual blocks.

\paragraph{Conditioning Construction}
The side information for the encoder's Context Producer includes absolute time embeddings for the
  \(j\) history steps, as well as a binary mask indicating which of the \(j\) history steps are valid
  versus left-padded, and dynamical covariates for the \(j\) history steps.
We additionally include static covariates, and similar to the encoder we project these from the
  data space to the Context Producer's spatial dimension to ease the task of matching the generated
  latent states to the observed data.
To summarize these inputs, we have:
\begin{itemize}
    \item \textbf{Temporal \& Covariate Information:}
          We gather absolute time embeddings and dynamic covariates aligned with the exact \((j+1)\) slots:
            time \(t\) for the summary slot, and \(t-j \dots t-1\) for the history slots.
    \item \textbf{Role Mask:}
          We employ a binary mask---set to \(0\) for the future summary slot and \(1\) for the history
            slots---which is projected through a learned linear layer.
          This explicitly instructs the residual blocks to treat the two modalities differently.
    \item \textbf{Padding Mask:}
          We use an additional binary mask representing valid versus zero-padded history steps, which is also
            passed through a learned embedding layer.
    \item \textbf{Static Information:}
          Similar to the decoder, static categorical or continuous covariates are mapped to the spatial
            attribute \(H_{\mathrm{seq}}\) and expanded identically across the temporal axis.
\end{itemize}

\paragraph{Gaussian Output}
The sequence \(\mathbf{H}_{\mathrm{in}}\) and side information \(\mathbf{S}\) are evaluated through
  the Context Producer described in \cref{app:context_producer_architecture}, returning a flattened
  representation vector in \(\mathbb{R}^{C \times H_{\mathrm{seq}}}\).
Finally, this vector is processed by a Gaussian MLP head to produce the mean \(\mu_t \in
  \mathbb{R}^M\) and diagonal log-variance \(\log \Sigma_t \in \mathbb{R}^M\) defining the
  variational posterior for \(\mathbf{z}_t\).

\subsection{Transition Architecture (Baselines)} \label{app:transition_architecture}
The Gaussian baseline transition model parameterizes \(p^^t_\psi(\mathbf{z}_t |
  \mathbf{z}_{t-j:t-1}, \mathbf{u}_{t-j:t})\) as a diagonal Gaussian, where the mean and log-variance
  are produced by a Context Producer followed by a Gaussian MLP head.

\paragraph{Input Preparation}
The transition model receives only the length-\(j\) latent history \(\mathbf{z}_{t-j:t-1}\) as
  input (no future summary or observation information).
Each latent vector \(\mathbf{z}_{t-i} \in \mathbb{R}^M\) is linearly projected to the Context
  Producer's spatial dimension \(H_{\mathrm{seq}}\), yielding the primary input tensor
  \(\mathbf{H}_{\mathrm{in}} \in \mathbb{R}^{H_{\mathrm{seq}} \times j}\).
Unlike the encoder, no padding is required: the transition model is only invoked for \(t > j\), so
  a full history of length \(j\) is always available.

\paragraph{Conditioning Construction}
Because the transition model has no future summary and no padded positions, its side information is
considerably simpler than the encoder's:
\begin{itemize}
    \item \textbf{Temporal Information:}
          Absolute time embeddings for the \(j\) history steps \(t-j, \ldots, t-1\) are gathered and, if
            dynamic covariates \(\mathbf{u}_{t-j:t}\) are present, concatenated along the feature dimension.
    \item \textbf{Mask Information:}
          No mask embeddings are used.
          The Context Producer receives a zero-dimensional mask tensor (i.e.\ \texttt{mask\_tot\_dim}~\(=
            0\)), so no role mask or padding mask is injected.
    \item \textbf{Static Information:}
          No static covariates are provided, since the transition operates entirely in the latent space and
            does not need to align with observation-level features.
\end{itemize}

\paragraph{Gaussian Output}
The Context Producer condenses the history into a flattened context vector in \(\mathbb{R}^{C
  \times H_{\mathrm{seq}}}\), which is passed to a Gaussian MLP head producing the mean
  \(\mu_\psi(\mathbf{z}_{t-j:t-1}) \in \mathbb{R}^M\) and diagonal log-variance \(\log
  \Sigma_\psi(\mathbf{z}_{t-j:t-1}) \in \mathbb{R}^M\).
The transition density is then
\begin{equation*}
    p^^t_\psi(\mathbf{z}_t | \mathbf{z}_{t-j:t-1}, \mathbf{u}_{t-j:t})
    = \mathcal{N}\!\left(\mathbf{z}_t;\; \mu_\psi(\mathbf{z}_{t-j:t-1}),\;
    \operatorname{diag}\!\left(\exp \log \Sigma_\psi(\mathbf{z}_{t-j:t-1})\right)\right).
\end{equation*}

\subsection{VHP Architecture} \label{app:vhp_architecture}
The variational hierarchical prior (VHP) for the initial latent states \(p_\eta(\mathbf{z}_{1:j} |
  \mathbf{z}_{-j+1:0})\) is implemented with two small Context Producers, one for the conditional
  prior \(p_\eta\) and one for the auxiliary variational posterior \(q_\Phi(\mathbf{z}_{-j+1:0} |
  \mathbf{z}_{1:j})\).
The architectures described below we implement with fewer layers and smaller hidden dimensions,
  given the simpler task of modeling only \(j\) time steps rather than the full sequence of length
  \(T\).
We refer to \cref{app:vhp} for the mathematical derivation.

\paragraph{Conditional Prior \texorpdfstring{\(p_\eta(\mathbf{z}_t | \mathbf{z}_{t-j:t-1})\)}{p\_eta}}
The conditional prior for \(t \leq j\) shares the same architecture as the Gaussian baseline
  transition described above: a Context Producer of temporal length \(j\) followed by a Gaussian MLP
  head producing \(\mu_\eta, \log\Sigma_\eta \in \mathbb{R}^M\).
Its side information consists of absolute time embeddings, optional dynamic covariates, and an
  embedded padding mask, since during the autoregressive generation of \(\mathbf{z}_{1}, \ldots,
  \mathbf{z}_{j}\), the conditioning history may include auxiliary latents \(\mathbf{z}_{-j+1:0}\)
  sampled from the VHP prior \(p(\mathbf{z}_{-j+1:0}) = \mathcal{N}(\mathbf{0}, \mathbf{I})\) or from
  the auxiliary posterior \(q_\Phi\).
Concretely, at step \(t \in \{1, \ldots, j\}\), the history \(\mathbf{z}_{t-j:t-1}\) is formed by
  concatenating the available auxiliary latents with any already-generated initial latents,
  left-padded with zeros when \(t < j\).
A binary padding mask indicating which positions are real versus padded is embedded through a
  learned linear layer and injected as mask side information, exactly as in the encoder
  (\cref{app:encoder_architecture}).

\paragraph{Auxiliary Posterior \texorpdfstring{\(q_\Phi(\mathbf{z}_{-j+1:0} | \mathbf{z}_{1:j})\)}{q\_Phi}}
The auxiliary posterior takes the encoder's initial latent samples \(\mathbf{z}_{1:j}\) and
  produces a diagonal Gaussian over all \(j\) auxiliary latents jointly.
Each latent \(\mathbf{z}_t\) for \(t \in \{1, \ldots, j\}\) is linearly projected from
  \(\mathbb{R}^M\) to \(\mathbb{R}^{H_{\mathrm{seq}}}\), forming an input sequence of shape
  \((H_{\mathrm{seq}}, j)\) for a second, smaller Context Producer.
This Context Producer uses the same residual-block architecture described in
  \cref{app:context_producer_architecture}, with temporal side information (absolute time embeddings
  and optional covariates for the first \(j\) steps) but no mask embeddings
  (\texttt{mask\_tot\_dim}~\(= 0\)), since the full sequence \(\mathbf{z}_{1:j}\) is always observed.
The flattened output in \(\mathbb{R}^{C \times H_{\mathrm{seq}}}\) is passed to a Gaussian MLP head
  that outputs the mean \(\mu_\Phi \in \mathbb{R}^{M \times j}\) and diagonal log-variance
  \(\log\Sigma_\Phi \in \mathbb{R}^{M \times j}\) for the \(j\) auxiliary latents.
Sampling is performed via the reparameterization trick: \[ \mathbf{z}_{-j+1:0} =
  \mu_\Phi(\mathbf{z}_{1:j}) + \exp\!
\left(\tfrac{1}{2}\log\Sigma_\Phi(\mathbf{z}_{1:j})\right) \odot \boldsymbol{\epsilon},
\qquad \boldsymbol{\epsilon} \sim \mathcal{N}(\mathbf{0}, \mathbf{I}).
\]

\section{Implementation Details}\label{app:implementation}
\subsection{Parameters}\label{app:parameters}
For clarity, we summarize the parameters being optimized in our model.
The full parameter set consists of five groups:

\begin{enumerate}
    \item \textbf{Encoder parameters} \(\phi\): all parameters of the encoder Context Producer,
          Future Summary module, and Gaussian MLP head
          (\cref{app:encoder_architecture}).
    \item \textbf{Decoder parameters} \(\theta\): all parameters of the decoder network
          (\cref{app:decoder_architecture}).
    \item \textbf{Transition parameters} \(\psi\): all parameters of the transition Context
          Producer and Gaussian MLP head (\cref{app:transition_architecture}).
    \item \textbf{VHP conditional prior parameters} \(\eta\): all parameters of the
          conditional prior Context Producer and Gaussian MLP head used to model
          \(p_\eta(\mathbf{z}_t | \mathbf{z}_{t-j:t-1})\) for
          \(t \in \{1, \ldots, j\}\) (\cref{app:vhp_architecture}).
    \item \textbf{VHP auxiliary posterior parameters} \(\Phi\): all parameters of the
          auxiliary posterior Context Producer and Gaussian MLP head used to model
          \(q_\Phi(\mathbf{z}_{-j+1:0} | \mathbf{z}_{1:j})\)
          (\cref{app:vhp_architecture}).
\end{enumerate}

\subsection{Optimization Techniques} \label{app:optimization}
It is a well-known difficulty of VAEs that the optimization of the ELBO may be viewed as a
  multi-objective optimization problem, where the two competing objectives are the reconstruction of
  the data and the regularization of the latent space via the
  prior~\cite{higgins2017beta,rezende2018taming}.
For example, in the most-general VAE setting, the ELBO may be written as
\begin{equation*}
    \text{ELBO}(\phi,\theta) = - \mathbb{E}_{ q_\phi(\mathbf{z} | \mathbf{x})}[\log p^^t_\theta(\mathbf{x} | \mathbf{z})] + \KL{q_\phi(\mathbf{z} | \mathbf{x})}{p(\mathbf{z})},
\end{equation*}
or in the distortion-rate form,
\begin{equation*}
    \text{ELBO}(\phi,\theta) = \mathcal{D}(\phi,\theta) + \mathcal{R}(\phi),
\end{equation*}
where the distortion term \( \mathcal{D}(\phi,\theta) := - \mathbb{E}_{ q_\phi(\mathbf{z} | \mathbf{x})}[\log p_\theta(\mathbf{x} | \mathbf{z})] \) measures the
reconstruction error, and the rate term \( \mathcal{R}(\phi) := \KL{q_\phi(\mathbf{z} | \mathbf{x})}{p(\mathbf{z})} \) measures the
complexity of the latent representation.\TODO{If unclear, we may put the specific rate distortion forms here}

It is commonplace in the VAE literature and the Deep SSM literature to employ a weighting term \(
  \beta\) on the rate term to balance the tradeoff between reconstruction and
  regularization~\cite{higgins2017beta, krishnan2017structured}, this is particularly important for
  DSSMs with learnable priors.
\AFTER{verify these citations and add more if needed.}
To emphasize learning a good reconstruction before regularization, it is common to anneal \( \beta
  \) from \( 0 \) to \( 1 \) during training.
To avoid confusing notation with the diffusion schedule, we denote this regularization term as
  \(\lambda\) instead of \(\beta\).

That is, we optimize the following objective:
\begin{equation*}
    \mathcal{L}_{\text{ELBO}}(\phi,\eta,\theta,\psi,\Phi; \mathbf{x}_{1 : T}, \mathbf{u}_{1 : T})
    = \mathcal{D}(\phi,\theta) + \lambda \mathcal{R}(\phi,\eta,\psi,\Phi),
\end{equation*}
where the distortion term \(\mathcal{D}\) includes the reconstruction loss, and the rate term \(\mathcal{R}\) collects the posterior
entropy, diffusion KL, any initialization likelihoods, and auxiliary posterior KL contributions.

We utilize cosine annealing schedules for \( \lambda \) in our experiments.
To find the optimal schedule, we use the hyperparameter \(\lambda_{\text{end}}\) to specify the
  ending value of \( \lambda \), and use the hyperparameter \(\lambda_{\text{warmup}}\) to specify
  the number of training iterations until \( \lambda \) reaches \( \lambda_{\text{end}} \).
Empirically, we find that choosing a good schedule is crucial for performance, hence we include
  these parameters in our hyperparameter sweep during training (see \cref{app:tuning}).
\AFTER{polish this section}

\subsection{Computing MC
    Estimates of ELBO Terms}
It still remains to detail the computation of the other ELBO terms, and to deal with the
  expectations in each term.

When any of the distributions involved are Gaussian, computing the KL divergences becomes
  straightforward as the Gaussian distribution admits closed-form expressions for likelihoods and KL
  divergences.
In the case of non-Gaussian distributions, we must use a Monte Carlo estimates to compute the
  various expectations.
To allow us to reduce the error of the Monte Carlo estimates, we draw several latent trajectory
  samples from our encoder during training.
Let the \( S \) trajectories drawn from the encoder during training be denoted by \(
  \mathbf{z}_{1:T}^{(s)} \sim q_\phi(\mathbf{z}_{1 : T} \mid \mathbf{x}_{1 : T}, \mathbf{u}_{1 : T})
  \).
We now detail the Monte-Carlo estimate for each term in the ELBO.

\subsubsection{Initialization Loss}
Following the decomposition in \cref{eq:vhp_loss}, we compute the initialization loss as the sum of
  the entropy term and the VHP loss.
Using the sampled latent trajectories \( \mathbf{z}_{1:T}^{(s)} \), we first compute the entropy estimate:
\begin{equation}
    \mathcal{L}_{\text{init}}^{\text{entropy, MC}}(\phi; \mathbf{x}_{1 : T}, \mathbf{u}_{1 : T})
    = \frac{1}{S} \sum_{s=1}^S \log q_\phi(\mathbf{z}_{1:j}^{(s)} | \mathbf{x}_{1:T}, \mathbf{u}_{1:T}).
\end{equation}
For the VHP loss, we sample auxiliary variables \( \mathbf{z}_{-j+1:0}^{(s)} \sim
  q_\phi(\mathbf{z}_{-j+1:0} | \mathbf{z}_{1:j}^{(s)}) \) for each trajectory \( s \).
The KL divergence term is computed analytically, assuming Gaussian distributions.
\begin{align}
    \mathcal{L}_{\text{VHP}}^{\text{MC}}(\phi,\eta,\Phi; \mathbf{x}_{1 : T}, \mathbf{u}_{1 : T})
     & = \frac{1}{S} \sum_{s=1}^S
    \Bigg[
    - \sum_{t=1}^j \log p_\eta(\mathbf{z}_t^{(s)} | \mathbf{z}_{t-j : t-1}^{(s)}) \nonumber \\
     & \qquad+ \KL{q_\phi(\mathbf{z}_{-j+1 : 0} | \mathbf{z}_{1 : j}^{(s)})}{p(\mathbf{z}_{-j+1 : 0})}
    \Bigg],
\end{align}
where the history \( \mathbf{z}_{t-j:t-1}^{(s)} \) in the reconstruction term draws from the auxiliary variables \( \mathbf{z}_{-j+1:0}^{(s)} \) when indices are non-positive.
The total initialization loss is
\begin{equation}
    \mathcal{L}_{\text{init}}^{\text{MC}} = \mathcal{L}_{\text{init}}^{\text{entropy, MC}} + \mathcal{L}_{\text{VHP}}^{\text{MC}}. \label{eq:init_loss_mc}
\end{equation}

\subsubsection{Sequence Reconstruction Loss}
For the reconstruction loss, we use the same sampled latent trajectories \( \mathbf{z}_{1:T}^{(s)}
  \sim q_\phi(\mathbf{z}_{1 : T} \mid \mathbf{x}_{1 : T}, \mathbf{u}_{1 : T}) \) as above.

Our decoder is Gaussian, and given most generally by
\begin{align*}
    p^^t_\theta(\mathbf{x}_t \mid \mathbf{z}_{\max(1,t-j+1) :
    t}, & \mathbf{u}_{\max(1,t-j+1) : t}) = \\
        & \mathcal{N}(\mu_\theta(\mathbf{z}_{\max(1,t-j+1) : t},
    \mathbf{u}_{\max(1,t-j+1) : t}), \Sigma_\theta(\mathbf{z}_{\max(1,t-j+1) : t},
    \mathbf{u}_{\max(1,t-j+1) : t})),
\end{align*} so the log-likelihood is available in closed form.
For \(t \leq j\), the conditioning window \(\mathbf{z}_{\max(1,t-j+1):t}\) contains fewer than
  \(j\) entries; in practice, the decoder left-pads the history with zeros to maintain a fixed input
  size, consistent with the architecture described in \cref{app:decoder_architecture}.

For trajectory sample \( \mathbf{z}_{1:T}^{(s)} \), we denote the decoder mean and covariance as
\begin{align*}
    \mu_\theta^{(s)}    & := \mu_\theta(\mathbf{z}_{\max(1,t-j+1) : t}^{(s)}, \mathbf{u}_{\max(1,t-j+1) : t}), \\
    \Sigma_\theta^{(s)} & := \Sigma_\theta(\mathbf{z}_{\max(1,t-j+1) : t}^{(s)}, \mathbf{u}_{\max(1,t-j+1) : t}).
\end{align*}

Using that
\[
    -\log \mathcal{N}(\mathbf{x}_t; \mu_\theta, \Sigma_\theta)
    = \frac{1}{2}
    \Big[
        (\mathbf{x}_t - \mu_\theta)^\top \Sigma_\theta^{-1} (\mathbf{x}_t - \mu_\theta)
        + \log \det(2\pi \Sigma_\theta)
        \Big],
\]
if we have diagonal covariance \( \Sigma_\theta \), this reduces to a weighted squared error plus a log
variance term and then

\begin{align*}
    - \log \mathcal{N}\big(\mathbf{x}_t;\, \mu_\theta, \Sigma_\theta\big)
     & = \frac{1}{2} \sum_{i=1}^D
    \left[
        \frac{(x_{t,i} - \mu_{\theta,i})^2}{\sigma_{\theta,i}^2}
        + \log(2\pi \sigma_{\theta,i}^2)
        \right],
\end{align*}
Thus we have

\begin{align}
    \mathcal{L}_{\text{recon}}(\phi,\theta; \mathbf{x}_{1 : T}, \mathbf{u}_{1 : T})
     & = \sum_{t=1}^T
    \mathbb{E}_{q_\phi(\mathbf{z}_{1:T} \mid \mathbf{x}_{1:T}, \mathbf{u}_{1 : T})}
    \Big[
        - \log p^^t_\theta\big(\mathbf{x}_t \mid \mathbf{z}_{\max(1,t-j+1) : t}, \mathbf{u}_{\max(1,t-j+1) : t}\big)
    \Big]\nonumber \\
     & \approx \frac{1}{S} \sum_{s=1}^S \sum_{t=1}^T
    \Big[
        - \log \mathcal{N}\big(\mathbf{x}_t;\, \mu_\theta^{(s)}, \Sigma_\theta^{(s)}\big)
    \Big]\nonumber \\
     & = \frac{1}{2S} \sum_{s=1}^S \sum_{t=1}^T \sum_{i=1}^D
    \left[
        \frac{(x_{t,i} - \mu_{\theta,i}^{(s)})^2}{\sigma_{\theta,i}^{(s)2}}
        + \log(2\pi \sigma_{\theta,i}^{(s)2})
    \right]\nonumber \\
     & := \mathcal{L}_{\text{recon}}^{\text{MC}}(\phi,\theta; \mathbf{x}_{1 : T}, \mathbf{u}_{1 : T}) \label{eq:recon_loss_mc}.
\end{align}

\subsubsection{Posterior Entropy}
Next, we compute the Monte Carlo estimate of the posterior entropy term.

\begin{align}
    \mathcal{L}_{\text{entropy}}(\phi; \mathbf{x}_{1 : T}, \mathbf{u}_{1 : T})
     & = \sum_{t=j+1}^T
    \mathbb{E}_{q_\phi(\mathbf{z}_{1:T} \mid \mathbf{x}_{1:T}, \mathbf{u}_{1 : T})}
    \Big[
        \log q^^t_\phi\big(\mathbf{z}_t \mid \mathbf{x}_{t:T}, \mathbf{z}_{t-j : t-1}, \mathbf{u}_{t-j : t}\big)
    \Big]\nonumber \\
     & \approx \frac{1}{S} \sum_{s=1}^S \sum_{t=j+1}^T
    \Big[
    \log q^^t_\phi\big(\mathbf{z}_t^{(s)} \mid \mathbf{x}_{t:T}, \mathbf{z}_{t-j : t-1}^{(s)}, \mathbf{u}_{t-j : t}\big)
    \Big]\nonumber \\
     & := \mathcal{L}_{\text{entropy}}^{\text{MC}}(\phi; \mathbf{x}_{1 : T}, \mathbf{u}_{1 : T}) \label{eq:post_entropy_mc}.
\end{align}

Again, for a Gaussian inference model, this term may be computed in closed form.

\subsubsection{Diffusion Loss}
The diffusion loss is given by \cref{eq:diff_loss}, and may be computed using the sampled latent
  trajectories \( \mathbf{z}_{1:T}^{(s)} \sim q_\phi(\mathbf{z}_{1 : T} \mid \mathbf{x}_{1 : T},
  \mathbf{u}_{1 : T}) \) as above.
We replace the expectation over the variational posterior with a Monte Carlo estimate,
\begin{align*}
    \mathcal{L}_{\widehat{\text{diff}}}(\phi,\psi; \mathbf{x}_{1 : T}, \mathbf{u}_{1 : T})
     & = \sum_{t=j+1}^T \mathbb{E}_{q_\phi(\mathbf{z}_{1:T} \mid \mathbf{x}_{1 : T}, \mathbf{u}_{1 : T})} \Big[
    \KL{q(\mathbf{z}_t^K | \mathbf{z}_t^0)}{p(\mathbf{z}_t^K)} \nonumber \\
     & \qquad+ \mathbb{E}_{\epsilon \sim \mathcal{N}(0,I), k \sim \text{Unif}\{1, \ldots, K\}}
    \left[
        K\tilde{w}_k ||F_\psi(\mathbf{z}_t^k, \mathbf{c}^{(s)}_t, k) - \+y(\mathbf{z}_t^0,k)||_2^2
        \right]
    \Big] \\
     & \approx \frac{1}{S} \sum_{s=1}^S \sum_{t=j+1}^T \Big[
    \KL{q(\mathbf{z}_t^K | \mathbf{z}_t^{0(s)})}{p(\mathbf{z}_t^K)} \nonumber \\
     & \qquad+ \mathbb{E}_{\epsilon \sim \mathcal{N}(0,I), k \sim \text{Unif}\{1, \ldots, K\}}
    \left[
    K\tilde{w}_k ||F_\psi(\mathbf{z}_t^{k(s)}, \mathbf{c}^{(s)}_t, k) - \+y(\mathbf{z}_t^{0(s)},k)||_2^2
    \right]
    \Big], \\
\end{align*}
where \( \mathbf{c}^{(s)}_t \) denotes the Markov history constructed from the sampled latent trajectory
\( \mathbf{z}_{1:T}^{(s)} \).
Note that for each of the sampled latent trajectories \( \mathbf{z}_{1:T}^{(s)} \), we must run the
  forward diffusion process to obtain the noisy latents \( \mathbf{z}_t^{k(s)} \) for \( k=1, \ldots,
  K \).
That is, for each time step \( t=j+1, \ldots, T \) and sample \( s=1, \ldots, S \), we draw unique
  noise samples and diffusion steps and compute the corresponding noisy latents.
To increase the amount of training signal to the diffusion model, we may draw multiple diffusion
  steps or noise samples per latent trajectory sample.
During typical diffusion model training, the clean data is known, so over the course of training
  the diffusion model sees many noise samples per data point.
Here, since the clean latent is itself sampled from a variational posterior, the diffusion model
  only sees one clean latent per data point per training iteration.
To compensate, we may draw \( S_k \) diffusion steps per latent trajectory sample \(
  \mathbf{z}_{1:T}^{(s)} \) (with corresponding noise samples), increasing the training signal to the
  diffusion model per iteration.
Additionally, we note that the first term in the diffusion loss is computed in closed form for
  Gaussian distributions.
\AFTER{show this. pretty much 0 for large enough k, but for small?}

The resulting Monte Carlo estimate of the diffusion loss is finally
\begin{align}
    \mathcal{L}_{\widehat{\text{diff}}}^{\text{MC}}(\phi,\psi; \mathbf{x}_{1 : T}, \mathbf{u}_{1 : T})
     & = \frac{1}{S} \sum_{s=1}^S \sum_{t=j+1}^T \Big[
    \KL{q(\mathbf{z}_t^K | \mathbf{z}_t^{0(s)})}{p(\mathbf{z}_t^K)} \nonumber \\
     & \qquad+ \frac{1}{S_k} \sum_{m=1}^{S_k} \mathbb{E}_{\epsilon \sim \mathcal{N}(0,I), k \sim \text{Unif}\{1, \ldots, K\}}
    \left[
    K\tilde{w}_k ||F_\psi(\mathbf{z}_t^{k(s,m)}, \mathbf{c}^{(s)}_t, k) - \+y(\mathbf{z}_t^{0(s)},k)||_2^2
    \right]
    \Big], \label{eq:diff_mc}
\end{align}
where \( \mathbf{z}_t^{k(s,m)} \) denotes the noisy latent at diffusion step \( k \) for time step \( t \),
sample \( s \), and diffusion sample \( m \).

\subsubsection{Full ELBO Objective}
Combining each of the above terms, we have the full ELBO objective
\begin{align}
    \mathcal{L}_{\text{ELBO}}^{\text{MC}}(\phi,\eta,\theta,\psi,\Phi; \mathbf{x}_{1 : T}, \mathbf{u}_{1 : T})
     & =
    \mathcal{L}_{\text{init}}^{\text{MC}}(\phi,\eta,\Phi; \mathbf{x}_{1 : T}, \mathbf{u}_{1 : T})
    + \mathcal{L}_{\text{recon}}^{\text{MC}}(\phi,\theta; \mathbf{x}_{1 : T}, \mathbf{u}_{1 : T}) \nonumber \\
     & \quad+
    \mathcal{L}_{\widehat{\text{diff}}}^{\text{MC}}(\phi,\psi; \mathbf{x}_{1 : T}, \mathbf{u}_{1 : T})
    + \mathcal{L}_{\text{entropy}}^{\text{MC}}(\phi; \mathbf{x}_{1 : T}, \mathbf{u}_{1 : T}). \label{eq:elbo_terms_mc}
\end{align}

\subsection{Training Algorithm}
We consider the case of training with multiple i.i.d.
time series of fixed length \( T \).
In this case we compute the ELBO over the full time series for each sequence in the minibatch, and
  average the ELBOs to get the minibatch ELBO.
We describe the training procedure in \cref{alg:training_nonwindowed}, and the sampling of latent
  trajectories in \cref{alg:sample_latent_paths}.

\begin{algorithm}[ht]
    \caption{Training for Diffusion State-Space Models} \label{alg:training_nonwindowed}
    \DontPrintSemicolon
    \SetAlgoLined
    \KwIn{Minibatch $\{(\mathbf{x}_{1:T}^{(b)}, \mathbf{u}_{1:T}^{(b)})\}_{b=1}^B$; untrained encoder $q_\phi$ (including $F_\phi$); initial prior $p_{\eta}$; auxiliary posterior $q_\Phi$; decoder $p^^t_\theta$; diffusion model $p^^t_\psi$; number of latent paths $S$; number of diffusion samples per path $S_k$; Markov order $j$}
    \KwOut{Updated parameters $(\theta, \phi, \psi, \eta, \Phi)$}

    \For{$b \leftarrow 1$ \KwTo $B$}{
        $\{\mathbf{z}_{1:T}^{(s,b)}\}_{s=1}^S,\; \{\mathbf{z}_{-j+1:0}^{(s,b)}\}_{s=1}^S \leftarrow \texttt{SampleLatentPaths}(\mathbf{x}_{1:T}^{(b)}, \mathbf{u}_{1:T}^{(b)}, F_\phi, q_\phi, q_\Phi, S, j)$\;
        Compute initialization loss $\mathcal{L}_{\text{init}}^{\text{MC},(b)}$ using \cref{eq:init_loss_mc} with $\{\mathbf{z}_{1:j}^{(s,b)}\}_{s=1}^S$ and $\{\mathbf{z}_{-j+1:0}^{(s,b)}\}_{s=1}^S$\;
        Compute reconstruction loss $\mathcal{L}_{\text{recon}}^{\text{MC},(b)}$ using \cref{eq:recon_loss_mc} with $\{\mathbf{z}_{1:T}^{(s,b)}\}_{s=1}^S$\;
        Compute posterior entropy term $\mathcal{L}_{\text{entropy}}^{\text{MC},(b)}$ using \cref{eq:post_entropy_mc} with $\{\mathbf{z}_{1:T}^{(s,b)}\}_{s=1}^S$\;
        Compute diffusion loss $\mathcal{L}_{\widehat{\text{diff}}}^{\text{MC},(b)}$ using \cref{eq:diff_mc} with $\{\mathbf{z}_{1:T}^{(s,b)}\}_{s=1}^S$ and $S_k$ diffusion samples per path\;
        Compute \(\lambda\) from annealing schedule\;
        $\mathcal{L}_{\text{ELBO}}^{\text{MC},(b)} \leftarrow
            \mathcal{L}_{\text{recon}}^{\text{MC},(b)} +
            \lambda \left(
            \mathcal{L}_{\text{init}}^{\text{MC},(b)} +
            \mathcal{L}_{\widehat{\text{diff}}}^{\text{MC},(b)} +
            \mathcal{L}_{\text{entropy}}^{\text{MC},(b)}\right)$\;
        \;
    }
    $\mathcal{L}_{\text{ELBO}}^{\text{MC}} \leftarrow \frac{1}{B}\sum_{b=1}^B \mathcal{L}_{\text{ELBO}}^{\text{MC},(b)}$\;
    Update $(\theta, \phi, \psi, \eta, \Phi)$ with gradient descent on $\mathcal{L}_{\text{ELBO}}^{\text{MC}}$\;
\end{algorithm}
\begin{algorithm}[H]
    \caption{\texttt{SampleLatentPaths}: Sampling of \( S \) latent trajectories} \label{alg:sample_latent_paths}
    \DontPrintSemicolon
    \SetAlgoLined
    \KwIn{Observations $\mathbf{x}_{1:T}$, covariates $\mathbf{u}_{1:T}$, future summary network $F_\phi$, encoder $q_\phi$, VHP auxiliary posterior $q_\Phi$, number of paths $S$, Markov order $j$}
    \KwOut{Latent trajectories $\{\mathbf{z}_{1:T}^{(s)}\}_{s=1}^S$ and auxiliary latents $\{\mathbf{z}_{-j+1:0}^{(s)}\}_{s=1}^S$}

    Compute future summaries $(\mathbf{h}_T, \ldots, \mathbf{h}_1) \leftarrow F_\phi(\mathbf{x}_{1:T}, \mathbf{u}_{1:T})$\;

    \For{$s \leftarrow 1$ \KwTo $S$}{
        \tcp{Step 1: Sample initial $j$ latents}
        \For{$t \leftarrow 1$ \KwTo $j$}{
            $k \leftarrow t - 1$ \tcp*{available history length}
            $m \leftarrow j - k$ \tcp*{missing history slots}
            Left-pad with zeros: $\mathbf{z}_{t-j:t-k-1}^{(s)} \leftarrow \mathbf{0}$\;
            Form length-$j$ history $\mathbf{z}_{t-j:t-1}^{(s)} \leftarrow (\mathbf{z}_{t-j:t-k-1}^{(s)},\; \mathbf{z}_{t-k:t-1}^{(s)})$\;
            Sample $\mathbf{z}_t^{(s)} \sim q^^t_\phi(\mathbf{z}_t \mid \mathbf{z}_{t-j:t-1}^{(s)}, \mathbf{h}_t)$\;
            Append $\mathbf{z}_t^{(s)}$ to path $s$\;
        }

        \tcp{Step 2: Sample auxiliary latents via VHP}
        Sample $\mathbf{z}_{-j+1:0}^{(s)} \sim q_\Phi(\mathbf{z}_{-j+1:0} \mid \mathbf{z}_{1:j}^{(s)})$\;

        \tcp{Step 3: Sample remaining latents autoregressively}
        \For{$t \leftarrow j+1$ \KwTo $T$}{
            $\mathbf{z}_{t-j:t-1}^{(s)} \leftarrow$ last $j$ latents from path $s$\;
            Sample $\mathbf{z}_t^{(s)} \sim q^^t_\phi(\mathbf{z}_t \mid \mathbf{z}_{t-j:t-1}^{(s)}, \mathbf{h}_t)$\;
            Append $\mathbf{z}_t^{(s)}$ to path $s$\;
        }
    }
    \Return $\{\mathbf{z}_{1:T}^{(s)}\}_{s=1}^S$, $\{\mathbf{z}_{-j+1:0}^{(s)}\}_{s=1}^S$\;
\end{algorithm}
We briefly describe the computational cost incurred by the diffusion model during training and
  sampling.
Under \cref{alg:training_nonwindowed}, we point out to the reader that the diffusion model receives
  an effective batch size of \(B_{\mathrm{eff}} = B \times (T-j) \times S \times S_k\), where \(B\)
  is the number of sequences in a batch, and \(T-j\) is the number of diffusion-parameterized
  transitions in each sequence.
Thus for training only a single evaluation of the diffusion model is required, assuming there is
  sufficient memory to handle the \(T - j\) factor in the batch.
As \(T\) scales, this factor increases necessitating an unavoidable cost in memory (unless gradient
  checkpointing~\citep{chen2016training} is used to trade off time for memory).
However, the larger batch size provides a stronger per-batch training signal to the diffusion model
  during training, amortizing the cost of training the diffusion model across more transitions.
The per-function-evaluation of the diffusion model depends on the size of the input tensor, which
  is \(B_{\mathrm{eff}} \times M \times j + 1\).
Therefore given enough memory, the diffusion-incurred time complexity for a forward pass of
  \cref{alg:training_nonwindowed} is \(O(g(M \times j+1))\) where \(g\) is the time complexity of a
  pass through the diffusion model.
It is only during sampling that we lose parallelism across time steps, in which case the total
  diffusion cost becomes \(O(K \times T \times g(M \times j+1))\), for \(K\) diffusion steps.

\section{Extended Related Works} \label{sec:extended_related_works}

\textbf{Deep state space models.}
Similarly to other Deep State Space Models (DSSMs)~\citep{girin2022dynamical}, our Diffusion-Driven
  State Space Model (DDSSM) parametrizes the transitions and emissions densities of a state space
  model with a neural network.
However, our model replaces the parametric (typically Gaussian) transition distribution of a DSSM
  with a diffusion model.
In particular, we extend the Deep Kalman Filter (DKF) of~\citet{krishnan2015deep,
      krishnan2017structured} to have diffusion transitions.
An immediate shortcoming of the usage of Gaussian transitions is the inability to model multi-modal
  transitions, which we verify in our simulation study in \cref{sec:simulation}.
\Citet{karl2016deep} observed the shortcoming of Gaussian transitions in the context of modeling
physical systems, arguing that the regularization provided by Gaussian transitions
harms reconstruction performance.

Several lines of work have attempted to make the DKF more expressive.~\citet{karl2016deep} proposed
  learning a more flexible transition by learning a deterministic transition of a stochastic
  variable.
That is, they learn \(\mathbf{z}_t = f_\psi(\mathbf{z}_{t-1}, \mathbf{u}_t, \beta_t)\) where
  \(f_\psi\) is a neural network and \(\beta_t\) is a stochastic parameter of the transition sampled
  from a simple distribution such as a Gaussian.
This ensures that the latent states are sampled from the transition distribution, forcing the
  encoder to receive gradients directly from the transition model.

\AFTER{
    Our work might be viewed from this lens as well, since the sampling process of a diffusion model
    can be viewed as a chain of deterministic denoising steps applied to a stochastic variable.
}
A separate track of work attempting to make the DKF more expressive has been to add auxiliary
  variables.
For example the Kalman VAE (KVAE)~\citep{fraccaro2017disentangled} introduces parameters
  \(\mathbf{a}_{1:T}\) and propose the generative model \(p(\mathbf{x}_{1:T}, \mathbf{a}_{1:T},
  \mathbf{z}_{1:T} | \mathbf{u}_{1:T}) = p(\mathbf{x}_{1:T} | \mathbf{a}_{1:T})p(\mathbf{a}_{1:T} |
  \mathbf{z}_{1:T})p(\mathbf{z}_{1:T} | \mathbf{u}_{1:T})\), relying on linear Gaussian
  \(p(\mathbf{a}_t | \mathbf{z}_t, \mathbf{u}_t)\) and \(p(\mathbf{z}_t | \mathbf{z}_{t-1},
  \mathbf{h}_t, \mathbf{u}_t)\) distributions.
\Citet{klushyn2021latent} extends this model to have nonlinear Gaussian transitions, proposing the Extended Kalman VAE (EKVAE).
The primary advantage to their method is that inference about the parameters \(\mathbf{a}_t\) use
  the filtering/smoothing equations of the extended Kalman filter~\citep{ribeiro2004kalman}, to allow
  the training objective for the transition model to be sample-based instead of likelihood-based.
It may be possible to extend the EKVAE and other deep SSMs to use diffusion transitions as well,
  but we leave this to future work: we emphasize that our contribution is not on the design of the
  generative model, but rather on the usage of diffusion models as flexible transition distributions.
\JACK{Using  diffusion for the transitions IS about the design of the generative model.}\JACK{It's not clear why one would want to extend EKVAE to diffusion transitions.}
\AFTER{We should empirically compare to EKVAEs.} \TODO{Can you speculate on some advantage of our approach over this?}
\TODO{You had written here something about a ``list of related works to mention.''}

Lastly, we note that several works have extended deep SSMs using normalizing flows
  (NFs)~\citep{papamakarios2021normalizing}.
While these admit exact-likelihoods, they are heavily constrained by the requirement of invertible
  neural network architectures, limiting applicability across many domains.
One recent example is from~\citet{chen2025normalizingpf}, who propose learning a differentiable
  particle filter~\citep{le2017auto}, allowing the transition and emission densities to be
  parameterized by NFs.
\AFTER{Must discuss and compare SMC based DSSMs.}
We note that diffusion models would be unsuitable for this approach, as the particle filtering
  algorithm requires many samples through the transition for each time step.
Another example is from~\citet{de2020normalizing}, who assume linear transition dynamics and
  parameterize the emission density with a NF.
This enables efficient inference using Kalman-style filtering/smoothing, but does not allow for
  flexible transition dynamics.
We mention~\citet{pmlr-v97-ziegler19a} as well, who consider a variety of NF-based approaches for
  modeling discrete sequences, but they do not consider the true posterior form of their proposed
  generative models.

\textbf{Models for temporal data which apply diffusion directly in observation
    space.
} \JACK{I've reorganized this a bit.
    Each time you bring up a work in the Related Works section, you need to make it clear (a) how it
      relates to ours and (b) why the difference matters, before moving on to some other paper.
    If you want to discuss multiple works together, it's okay to hold off on details until later.
    But if you do that, you should still open with a high-level summary of the relation.
}
Several lines of work model temporal data by applying diffusion methods directly in observation
  space, rather than incorporating a latent dynamic process.
The advantage to this approach is that it does not require learning a latent representation of the
  data, and does not require performing inference about the latent state.
For example,~\cite{rasul2021autoregressive} employs autoregressive forecasting using a diffusion
  model, by conditioning the diffusion model on a summary of all previous time steps of the series.
To create a summary, an RNN is run over the history of the time series, providing a single
  conditioning variable as an input to their diffusion model.
The RNN summary is similar in spirit to the latent state of our model; in our model the
  conditioning set \(\mathbf{c}_t\) fully summarizes the history of the time series up to the time
  \(t\).
The advantage of our approach is that we are able to operate the diffusion model fully in
  latent-space,\TODO{unclear what this means} which is by construction well-suited to modeling the
  dynamics of the time series.
Additionally, the latent space can be designed to be low-dimensional, which can make training and
  sampling more efficient.
Another advantage is that our latent representation may provide a more interpretable representation
  of the dynamics of the time series, whereas it is unclear how to interpret the RNN summary
  of~\citet{rasul2021autoregressive}.
Another well known work is by~\citet{tashiro2021csdi}, who perform imputation using diffusion
  models by learning to denoise imputation targets while conditioning on observed data.
Their method operates on fixed length windows of a time series, and handles forecasting by setting
  imputation targets to be future time steps through a conditioning mask.
By changing the mask, their method can sample many timesteps into the future by conditioning on a
  fixed length history.
One drawback of this approach is that the model cannot utilize past data beyond the fixed length
  window which is determined at training time.
This is contrary to~\citet{rasul2021autoregressive} or our method which attempt to infer the
  necessary parameters to accurately sample the next time step.
More importantly,~\citet{tashiro2021csdi} requires their diffusion model to jointly model the full
  conditional distribution of the entire future given the entire history in a fixed window, which is
  expensive to learn and sample from.
We note that we have designed our diffusion model architecture to match~\citet{tashiro2021csdi}
  closely, as their previously state-of-the-art performance suggests that their architecture is
  well-suited to modeling temporal data.
We detail the differences between our architecture and theirs in \cref{app:unet_architecture}.


\textbf{Latent diffusion models for static data.} \TODO{closely compare vahdat and shmakov, and carefully consider differences after a second pass.}
While diffusion models are already computationally expensive in the static setting---motivating
  latent diffusion models~\citep{rombach2022high}---the challenge is amplified for time series, where
  the model must capture both high-dimensional structure and temporal dynamics.
Our work can be considered a temporal extension of the diffusion models of~\citet{vahdat2021score}
  and~\citet{shmakov2023end}, who train a diffusion model end-to-end in the latent space of a VAE.
Like~\cite{vahdat2021score}, we design the training objective of the diffusion model to account for
  the fact the model is being trained concurrently with the VAE.
However, we employ preconditioning~\citep{karras2022elucidating} of the denoising model to reduce
  the variance of the gradients across noise levels.
This is important since the gradients of the diffusion model are used to train the VAE.
Details of our approach to concurrent training are described in
  \cref{app:faithful_diff_elbo_proof}.

\textbf{Latent diffusion models for temporal data. (WIP)} \TODO{Restate modeling/compute advantages over diffusion models for temporal data in observation space.}
\TODO{remove WIP from the title of this section, and expand on this section.}
\TODO{write individual paper reviews, and then we can organize by subject.} \Citet{qian2024timeldm} propose a latent diffusion model framework for unconditional time-series generation, which they call TimeLDM.
They employ a VAE training objective, where the encoder accepts an entire time series as input, and
  then produces an entire latent time series as output, i.e \(\mathbf{z}_{1:T} =
  \text{Encoder}(\mathbf{x}_{1:T})\).
However, there is no mechanism in their model to capture any temporal structure within
  \(\+z_{1:T}\).
The decoder is responsible for reconstructing the entire time series from the latent time series,
  and the diffusion model is trained to model the distribution of time-series examples in latent
  space.
We note that their method cannot perform forecasting, and they train in a two-stage manner where
  the VAE is trained separately from the diffusion model.
In comparison, our method does capture temporal structure in the latent space, as the latent state
  at time \(t\) is conditioned on the latent state at time \(t-1\).
Additionally, our model can perform any tasks available to a standard state-space model, including
  unconditional generation.

\Citet{suh2025timeautodiff} proposes a latent diffusion modeling framework to
unify forecasting, imputation, unconditional generation, and conditional generation of time series.
Their autoencoder similar to~\citet{qian2024timeldm} outputs an entire latent time series given an
  entire time series as input.
The decoder however aims to reconstruct a target subset of observations given the latent time
  series, a conditioning set from the observations, and a mask to indicate which inputs are
  targets/conditional.
The diffusion model is then trained to model the distribution of the latent time series,
  conditioned on the conditioning set and mask.
This is to say that the diffusion model does not operate independently of the data, contrary to
  our method where the diffusion model lives fully in the latent space, conditioning only on latent
  features.
The strongest similarity which~\citet{suh2025timeautodiff} has to our method is that their method
  constructs a full ELBO training objective to maximize the log-likelihood of the data, where the
  loss of the diffusion model is a component of the ELBO.
Of interest is the observation that end-to-end training was detrimental to their model, instead
  they prefer first training the autoencoder and then training the diffusion model separately.
This claim however is not validated within their paper, and it is unclear how they attempted to
  train end-to-end.
As an artifact of training in stages, their variational posterior is regularized towards a standard
  Gaussian distribution.
If training jointly, their objective instead would regularize the variational posterior towards the
  distribution of the diffusion model.\TODO{Could explain why end-to-end training works for us}
\Citet{liu2024align} and~\cite{feng2024latent}
give other approaches to time-series diffusion modeling in the latent space of a
pre-trained autoencoder. We note that~\citet{feng2024latent} applies regularization of the latent space during the training of the autoencoder to a standard Gaussian distribution to avoid learning
an unregularized space.
\AFTER{Broad outline of what to say here:
    \begin{itemize}
        \item We can revise our belief about \(z_t\) as we obtain future information.
        \item We consider control inputs/covariates.
        \item Need to consider how uncertainty is different between our approach and other approaches.
              If a non-variational autoencoder is used, then there is no explicit modeling of uncertainty.
    \end{itemize}
}


\end{document}